\documentclass[twoside]{article}

\usepackage[accepted]{aistats2025}

\usepackage{parskip}
\usepackage[utf8]{inputenc} 
\usepackage[T1]{fontenc}    
\usepackage{hyperref}       
\hypersetup{colorlinks,allcolors=blue,linktocpage=true}
\usepackage{url, cleveref}            
\usepackage{booktabs}       
\usepackage{amsfonts}       
\usepackage{nicefrac}       
\usepackage{microtype}      
\usepackage{xcolor}         
\usepackage{amssymb}
\usepackage{amsmath}
\usepackage{mathtools}
\usepackage{amsthm, bbm, bm}
\usepackage{thm-restate, thmtools, enumitem}
\usepackage{algorithm}
\usepackage{algpseudocode}
\usepackage{soul}
\usepackage{multicol}

\usepackage{verbatim}

\usepackage{multirow}
\usepackage{siunitx} 

\usepackage{subcaption}

\newtheorem{lemma}{Lemma}

\newtheorem{theorem}{Theorem}

\newtheorem{definition}{Definition}

\newtheorem{prop}{Proposition}

\DeclarePairedDelimiterX\brk[2]{\langle}{\rangle}{#1\,,\,#2} 
\DeclarePairedDelimiterX\Set[2]{\{}{\}}{#1 \;;\; #2} 

\newcommand{\argmax}{\operatornamewithlimits{arg\,max}}
\newcommand{\argmin}{\operatornamewithlimits{arg\,min}}
\newcommand\independent{\protect\mathpalette{\protect\independenT}{\perp}}
\def\independenT#1#2{\mathrel{\rlap{$#1#2$}\mkern2mu{#1#2}}}
\newcommand\indep\independent

\newcommand{\EE}{\mathbb{E}}

\newcommand{\NN}{\mathbb{N}}

\newcommand{\PP}{\mathbb{P}}

\newcommand{\RR}{\mathbb{R}}

\newcommand{\Bcal}{\mathcal{B}}

\newcommand{\Ecal}{\mathcal{E}}
\newcommand{\Fcal}{\mathcal{F}}
\newcommand{\Gcal}{\mathcal{G}}
\newcommand{\Hcal}{\mathcal{H}}
\newcommand{\Ical}{\mathcal{I}}

\newcommand{\Pcal}{\mathcal{P}}

\newcommand{\Bal}{\begin{align}}
\newcommand{\Eal}{\end{align}}
\newcommand{\Beq}{\begin{equation}}
\newcommand{\Eeq}{\end{equation}}
\newcommand{\Bit}{\begin{itemize}}
\newcommand{\Eit}{\end{itemize}}
\newcommand{\Ben}{\begin{enumerate}}
\newcommand{\Een}{\end{enumerate}}

\newcommand{\Ba}{\begin{array}}
\newcommand{\Ea}{\end{array}}

\newcommand{\Bvec}{\left(\begin{array}{c}}
\newcommand{\Evec}{\end{array}\right)}

\newcommand{\Bmat}{\left(\begin{array}}
\newcommand{\Emat}{\end{array}\right)}

\newcommand{\Bol}{\begin{outline}}
\newcommand{\Eol}{\end{outline}}

\usepackage[round]{natbib}

\begin{document}
%
\runningtitle{Minimax Optimal Good Arm Identification for Nonparametric Multi-Armed Bandits}

%

\twocolumn[\aistatstitle{Reward Maximization for Pure Exploration: Minimax Optimal Good Arm Identification for Nonparametric Multi-Armed Bandits}
\aistatsauthor{Brian Cho \And Dominik Meier \And  Kyra Gan \And Nathan Kallus}
\aistatsaddress{ Cornell Tech \And  Cornell Tech \And Cornell Tech \And Cornell Tech} ]

\begin{abstract}
    In multi-armed bandits, the tasks of reward maximization and pure exploration are often at odds with each other. The former focuses on exploiting arms with the highest means, while the latter may require constant exploration across all arms. In this work, we focus on \emph{good arm identification} (GAI), a practical bandit inference objective that aims to label arms with means above a threshold as quickly as possible. We show that GAI can be efficiently solved by combining a \emph{reward-maximizing} sampling algorithm with a novel nonparametric anytime-valid sequential test for labeling arm means. We first establish that our sequential test maintains error control under highly nonparametric assumptions and asymptotically achieves the minimax optimal $e$-power, a notion of power for anytime-valid tests. Next, by pairing regret-minimizing sampling schemes with our sequential test, we provide an approach that  achieves minimax optimal stopping times for labeling arms with means above a threshold, under an error probability constraint $\delta$. Our empirical results validate our approach beyond the minimax setting, reducing the expected number of samples for all stopping times by at least 50\% across both synthetic and real-world settings.
\end{abstract}

\section{Introduction}
The multi-armed bandit (MAB) framework is a canonical model for sequential decision-making under uncertainty.
In MAB, the learner selects from
a finite set of actions (\emph{arms}), often characterized by the expected value of their \emph{rewards}.
Common objectives in MAB include  maximizing the expected cumulative rewards \citep{slivkins2024introductionmultiarmedbandits} and 
pure exploration \citep{bubeck2010pure}.
The former aims to specify a sampling policy that minimizes \emph{regret}, 
the difference between expected realized
and optimal
rewards.
The latter often focuses on inferring arm mean properties, e.g.,
identifying the arm with the highest mean in
best arm identification \citep{Audibert2010BestAI}, 
or labeling arm means as above/below a threshold in
threshold identification (THR, \citealt{locatelli2016optimal}).
Unlike reward maximization, 
sampling schemes for pure exploration problems disregard the cumulative rewards and focus solely on rapidly collecting evidence to reach a $\delta$-correct answer as quickly as possible. 

While both objectives have been studied separately, there is growing interest in studying them
together. In settings such as education and health \citep{trading_off_healthcare_education}, users often participate in adaptive studies to 
improve their outcomes, aligning with the reward maximization objective.
In contrast, the experimenter running the study often wants reliable inference on the set of available actions, aligning with pure exploration objectives. 
While these objectives are often at odds, this need not be the case depending on the pure exploration objective at hand. 

In this work, we focus on the pure exploration problem of \emph{good arm identification} (GAI) \citep{kano2019good} in the setting of nonparametric, bounded bandits. We call an arm "good" if its mean is above a pre-specified threshold value that represents some level of satisfactory performance. GAI seeks to identify arms above the threshold as quickly as possible, corresponding to many practical use-cases. For example, in medical settings where acquiring patients is costly \citep{medical_example, Genovese863}, the experimenter may want to find some treatment with satisfactory effect as quickly as possible rather than  classifying all considered treatments as satisfactory/unsatisfactory (i.e., THR). 


We provide an approach for GAI that leverages horizon-free regret-minimizing algorithms \citep{moss_anytime} as a subroutine with a novel anytime-valid (AV) test for labeling arms as good or bad. Under nonparametric assumptions that only require mean stationarity, our approach guarantees nonasymptotic, $\delta$-level error control for any arm label being incorrect. Additionally, we prove that our approach asymptotically achieves the minimax optimal stopping time for (i) labeling any number of good arms, and (ii) labeling all arms as above/below the threshold. In the case of identifying one good arm, our approach solves the pure exploration question optimally while also providing \emph{maximum benefit} (e.g., providing good treatments at an optimal rate to study participants).

\paragraph{Contributions} 

Our contributions are two-fold. First, we provide a novel $e$-process-based sequential AV test based on the Bernoulli likelihood ratio. Our approach provides time-uniform Type I error control for nonparametric arm distributions, even under mild nonstationarity. We show our test asymptotically achieves minimax $e$-power \citep{vovk2024efficiency}, a measure of power in the sequential testing setting. Second, we pair our test with a modified version of regret-optimal reward maximization schemes. We demonstrate our approach achieves the asymptotically optimal stopping time for identifying any number of good arms. When seeking to identify one good arm as quickly as possible, our sampling scheme reduces to a horizon-free regret-optimal sampling strategy, aligning pure exploration with regret minimization.


\paragraph{Outline} In Section \ref{sec_2:related_work}, we review related work, and we provide our bandit setting, in Section \ref{sec_3:prob_formulation}. Section \ref{sec_4:av_stopping_rules} introduces our novel sequential AV test and its guarantees. In Section \ref{sec_5:sampling_scheme}, we provide our approach for GAI, and prove key theoretical properties regarding its stopping time. In Section \ref{sec_7:empirical_results}, we provide extensive simulations on both synthetic and semi-synthetic data.

\section{Related Work}
\label{sec_2:related_work}
We contextualize our work by discussing (i) pure exploration bandit problems, (ii) AV tests for pure-exploration bandit problems, and (iii)  works discussing tradeoffs between inference and reward maximization.  

\paragraph{Pure Exploration Bandit Problems} Pure exploration bandit problems can be divided into two distinct categories: \emph{fixed-confidence} \citep{garivier2016optimal} and \emph{fixed-budget}  \citep{locatelli2016optimal} approaches. The former aims to minimize the number of samples before providing an answer under an error probability constraint. The latter aims to maximize confidence with a fixed budget of samples. We focus the fixed confidence setting, where an algorithm is defined by (i) a policy for sampling arms and (ii) stopping time(s) $\tau$ for providing an answer to the 
identification task. While existing works in this setting have focused on sampling policies \citep{kano2019good, jamieson}, they neglect the necessity of optimizing the stopping rules. These stopping rules are based on AV testing procedures, which we discuss below.

\paragraph{Anytime-Valid Sequential Tests} In the fixed-confidence setting, numerous works \citep{garivier2016optimal, kaufmann2021mixture} provide statistical guarantees by leveraging AV testing methods. These methods maintain Type I error rates over an potentially infinite experimental horizon, ensuring valid error control at data-adaptive stopping times. Existing works, however, have only shown optimality stopping time results
for well-studied parametric distributions, such as exponential families \citep{juneja2019samplecomplexitypartitionidentification}. In such cases,  the optimal sequential test is known to be based on likelihood ratio thresholding \citep{Wald1945SequentialTO}. In nonparametric cases that allow for distributional nonstationarity, optimality results are currently unavailable. For nonparametric reward distributions over bounded support, we leverage the testing-by-betting framework \citep{waudbysmith2022estimating} and $e$-processes \citep{ramdas2023gametheoretic, Vovk_2021} to propose a novel AV test. We show that our test provides Type I error control uniformly across time and obtains \emph{minimax} optimal $e$-power, the notion of statistical power for AV tests \citep{vovk2024efficiency} in the nonparametric setting.

\paragraph{Trading off Rewards and Inference} Existing works mainly focus on sampling schemes that balance reward maximization with inference on either arm means \citep{trading_off_healthcare_education} or their differences \citep{mab_experiment_design, liang2023experimental}. These works take the sample size as fixed, focus on estimation/inference on scalar quantities, and propose alternative sampling schemes that balance their objectives. Our work differs by letting our experiment size be determined by the time until we reach an answer regarding the bandit instance, using well-studied regret-optimal sampling policies with novel testing procedures, and focusing on GAI, a pure exploration objective. The closest related work to our aims is \cite{degenne2019bridginggapregretminimization}, which leverages ill-calibrated UCB algorithms for best arm identification in finite time. In contrast, our work uses tightly calibrated regret-minimizing approaches to obtain \emph{asymptotically optimal} stopping times for GAI.







\section{Problem Formulation}
\label{sec_3:prob_formulation}
We consider a nonparametric, nonstationary generalization of the classical multi-armed bandit problem with $K$ arms. At each time step $t$, the learner chooses to pull an arm  $A_t \in [K]$, where $[K] = \{1,...,K\}$. The choice of arm is defined by the sampling policy $\pi_t: (X_i, A_i)_{i=1}^{t-1} \rightarrow \Delta^K$, where $\Delta^K$ presents the probability simplex over the $K$ arms. The learner then observes feedback $X_{t} \in [0,1]$. 
We use $\Fcal_{t}$ as the canonical filtration at time $t$, i.e., $\Fcal_{t} = \sigma((A_i, X_i)_{i=1}^{t})$, with $\Fcal_0$ as the empty sigma field. 

Beyond boundedness, we assume that when conditioned on 
any possible previous observations $\Fcal_{t-1}$ and any action $A_t$ at time $t$, the conditional arm means $\bm\mu\equiv\{\mu_1, ..., \mu_K\}\in[0,1]^K$ remain constant:
$$
\forall t \in \NN, \ \forall a \in[K],
\ \ \EE[X_t|A_t=a, \Fcal_{t-1}] = \mu_a. 
$$
Let $\Pcal(\mu_a)$ be the set of all \emph{distributional sequences} on $[0,1]^\infty$ such that $\EE[X_t| A_t=a, \Fcal_{t-1}] = \mu_a$.
Then,
each arm $a$ is associated with a \emph{distribution sequence} $P_a\in \Pcal(\mu_a)$. 
This assumption covers common settings where the reward distribution is either stationary, but also includes settings where arm distributions undergo exogenous changes over time, endogenously changes based on the realized trajectory, or both. 
We assume nonparametric reward distributions, allowing for
continuous, discrete, or mixture distributions both across arms $a \in [K]$ and within an arm $a$ across time
$t$. 

Throughout the paper, we denote 
$\xi \in (0,1)$ the threshold value. We use $P \in \Pcal(\bm\mu)$ to denote the bandit instance, where $\Pcal(\bm\mu) \equiv \cap_{a \in k} \Pcal(\mu_a)$. 


\textbf{Threshold and Good Arm Identification}\;
Both problems
share similar setups, where GAI can be seen as a special case of THR.
We label an arm $a$ as good if its mean $\mu_a$ is greater than the threshold $\xi$. For any bandit instance in $\Pcal(\bm\mu)$,
we denote the set of true good arms as $\Gcal_{\bm\mu} = \{a \in [K]: \mu_a  > \xi\}$, and its complement $\Bcal_{\bm\mu} = [K]\setminus \Gcal_{\bm\mu}$ as the set of true bad arms. 

At each time $t$, the learner maintains two candidate sets: $\Gcal_t$ for good arms and $\Bcal_t$ for bad arms.
We denote $\tau_a$ as the first time arm $a$ is labeled as either good or bad, and 
$\tau_{\Gcal, i} = \inf\{t \in \NN: |\Gcal_t| = i\}$ as
the first time 
$i$ arms are labeled as good. 
Finally, let 
$\tau_{\text{stop}} = \inf\{t \in \NN: \Gcal_{t} \cup \Bcal_{t} = [K]\}$ be the first time all arms are labeled good or bad.

We focus on the \emph{fixed confidence} setting. Given a fixed error rate $\delta$,
the goal is to ensure that the probability of mislabeling \emph{any} arm across time
is
at most  $\delta$.
\begin{definition}[$\delta$-level Error Control]
\label{defn:error_control}
    We say that an algorithm provides $\delta$-level error control if the probability of mislabeling any arm at any time
    is at most $\delta$, i.e.,
    \begin{equation}
        \PP(\exists \tau_a \ s.t. \ \{\Gcal_t \not\subseteq \Gcal_{\bm\mu}\} \cup \{\Bcal_t \not\subseteq \Bcal_{\bm\mu}\}  ) \leq \delta.
    \end{equation}  
\end{definition}

THR and GAI differ by minimizing different stopping times under $\delta$-level error control. 
In THR, the goal is to label all arms as good or bad
as quickly as possible. 
This is equivalent to minimizing the expected stopping time of the overall experiment $\tau_{\text{stop}}$. In contrast, GAI presents a more difficult problem by minimizing the stopping time for each good arm identified. Given a bandit instance in $\Pcal(\bm\mu)$, let $G_{\bm\mu} = |\Gcal_{\bm\mu}|\leq K$ denote the number of good arms. GAI attempts to minimize the expected labeling time $\tau_{\Gcal, i}$ for all $i \in \{1,...G\}$ without knowing the number of good arms $G$ in advance. We include a notation table in Appendix \ref{app:notation}.

Algorithms for 
THR and GAI 
are composed of two distinct components: (i) $\Fcal_t$-measurable sequential tests used for labeling arms at each time $t$, and (ii) a sampling policy $\pi_t: (A_i, X_i)_{i < t} \rightarrow \Delta^K$.
In this work,
we provide
a novel sequential test, and
show that when paired with
popular regret-minimizing sampling scheme, 
we achieve minimax optimality in labeling times.
We first introduce our novel sequential test in Section~\ref{sec_4:av_stopping_rules} below.

\section{Anytime-Valid Good Arm Tests}
\label{sec_4:av_stopping_rules}
A natural choice for achieving 
$\delta$-level error control in
THR and GAI is to use anytime-valid sequential tests. 
For each arm $a$, AV tests control the error probability of mislabeling an arm
while allowing for repeated testing across all $t \in \NN$.
An arm $a$ is labeled as good when the chosen AV test rejects the composite hypothesis 
$\Hcal^-_a=\{\Pcal(\mu_a): \mu_a \leq \xi\}$.
Likewise, an arm $a$ is labeled as bad when 
the test rejects 
$\Hcal^+_a= \{\Pcal(\mu_a):\mu_a > \xi\}$.

In this section, we first define sequential tests based on $e$-process (Proposition~\ref{prop:delta_error_control}) and 
then define
$e$-power (Definition~\ref{defn:e_power}), the natural notion of power for the sequential setting. We then provide a universal representation for $e$-processes under our nonparametric assumptions (Definition~\ref{defn:test_martingale}). Lastly, we introduce our novel predictable plug-in sequential test that builds upon this framework (Definition~\ref{defn:predicable_maximin}), and justify our choice through minimax optimality results on $e$-power (Theorems~\ref{thm:minmax_opt_e_power} and \ref{thm:minmax_opt_gen_bern}). 

\subsection{Sequential Testing with e-Processes}
The $e$-process
is a nonnegative process that serves as a \emph{measure of evidence} against a null hypothesis set of distributions $\Hcal$ \citep{ramdas2023gametheoretic}. 
\begin{definition}[$e$-Process]\label{defn:e_process}
    We define a sequence of random variables $E = (E_t)_{t \in \NN}$ as a process if for all $t \in \NN$, $E_t$ is $\Fcal_t$-measurable. We say that the process $E$ is an $e$-process for null hypothesis set $\Hcal$ if (i) $E$ is nonnegative and (ii) $\sup_{P \in \Hcal}\sup_{\tau}\EE_P[E_\tau] \leq 1$, where first supremum is taken over all possible distributions in the null set $\Hcal$, and the second supremum is taken over all (potentially infinite) stopping times $\tau$. 
\end{definition}

The canonical example of an $e$-process in the sequential testing literature is the product of likelihood ratios \citep{Wald1945SequentialTO}, which satisfies the requirements in Definition \ref{defn:e_process}. To understand the role of the $e$-process as a running measure of evidence, we first provide the sequential test associated with $e$-processes.

\begin{prop}[Anytime-Valid Test using $e$-processes]\label{prop:delta_error_control}
    If $P \in \Hcal$, then for $e$-process $E = (E_t)_{t\in\NN}$ w.r.t. the null $\Hcal$, 
    $$\PP_P(\exists t \in \NN: E_t \geq 1/\delta) \leq \delta.$$
    The implied AV test $T_t(E, \delta) = \mathbf{1}[E_t \geq 1/\delta]$, where $T_t(E, \delta) = 1$ represents rejection of the null at time $t$, provides $\delta$-level error control across all times $t \in \NN$ (or equivalently any (potentially infinite) stopping time $\tau$). 
\end{prop}

The AV-test associated with $e$-process $E$ rejects the null set $\Hcal$ when $E$ is at least as large as $1/\delta$. Intuitively, high values of $E$ represent larger amounts of evidence against the null $\Hcal$, while small values of $E$ represent minimal evidence for rejecting $\Hcal$. In cases where the true distribution $P \not\in \Hcal$, one would hope that the value of $E_t$ grows quickly with respect to $t$ in order to reject the null as quickly as possible. The concept of $e$-power quantifies this notion precisely, and serves as a natural measure of optimality when $P \not\in \Hcal$.
\begin{definition}[$e$-power]\label{defn:e_power}
    For an $e$-process $E = (E_t)_{t \in \NN}$ with respect to null hypothesis set $\Hcal$, we define $e$-power of any  $P \not\in \Hcal$ as $\EE_{P}[\log(E_t)]/t$. 
\end{definition}
The $e$-power measures the logarithmic \emph{growth rate} of an $e$-process $E$, normalized by the number of observations $t$. Larger values of $e$-power of a distribution $P \not\in \Hcal$ intuitively correspond to faster expected rates at which we accumulate evidence against the null.

\subsection{Universal Representation via Test Supermartingales}

While Definition \ref{defn:e_process} provides the requirements for a process $E$ to be an $e$-process,
it does not explain how to construct one.
To develop an $e$-process for our use cases, we first focus on a subclass of $e$-processes called \emph{nonnegative test (super)martingales} (NM). Nonnegative test (super)martingales are an important subset of $e$-processes: any anytime-valid procedure must utilize nonnegative test (super)martingales to be admissible
\citep{ramdas2022admissible} (Lemma~\ref{lem:domination} in Appendix~\ref{app:proof_thm_minmax_opt_e_power}). 
We focus specifically on the class of NMs to construct our labeling tests. Under our nonparametric assumptions in Section \ref{sec_3:prob_formulation}, NMs have a universal representation, provided below:

\begin{definition}[Test Supermartingales for One-Sided Mean Testing,
\citealt{waudbysmith2022estimating}]\label{defn:test_martingale}
    Let $b \in [0,1]$ be a constant. For arm $a$ and null hypothesis sets $\Hcal^+_a = \{\Pcal(\mu_a): \mu_a > \xi\}$ and  $\Hcal^-_a = \{\Pcal(\mu_a): \mu_a \leq \xi\}$, we define a class of $e$-processes denoted as $\Ecal(\Hcal^+_a, b)$ and $\Ecal(\Hcal^-_a, b)$, respectively:
    \begin{align}\label{eq:one_sided_test}
        \Ecal(\Hcal^-_a, b) \equiv \{E_a = \left(E_t^a(\lambda, \xi)\right)_{t \in \NN}\}_{\lambda \in [0, b/\xi]}, 
        \\
        \quad
        \Ecal(\Hcal_a^+, b) \equiv \{E_a = \left(E_t^a(\lambda, \xi)\right)_{t \in \NN}\}_{\lambda \in [-b/(1-\xi), 0]},
    \end{align} 
    where $E_t^a(\lambda, \xi) = \prod_{i: A_i = a} (1+  \lambda(X_i - \xi) )$ and $\lambda$ is a $\Fcal_{t-1}$-measurable univariate parameter. 
\end{definition}

The truncation constant $b \in [0,1]$ ensures that our test martingales are strictly non-negative and satisfy Definition \ref{defn:e_process}. The restrictions on the sign of $\lambda$ intuitively align with the null hypotheses $\Hcal$ each test martingale wishes to reject. When $\mu_a > \xi$, $X_i - \xi$ is positive in expectation. Multiplied by a positive scalar $\lambda$, this ensures that $E_a \in \Ecal(\Hcal_a^-, b)$, the evidence against $\Hcal_a^-$, grows large. The same intuition holds for the sign restriction on $\lambda$ for $E_a \in \Ecal(\Hcal_a^+, b)$.

When $b = 1$ (i.e., we do not pose additional restriction on $\lambda$ other than $E_a$ is nonnegative), 
Definition~\ref{defn:test_martingale} provides a
\emph{universal} representation of any nonnegative test martingale (Proposition 3 of \citealt{waudbysmith2022estimating}) used to test whether arm means are above 
a threshold $\xi$. Equivalently, any test martingale for one-sided mean testing under our assumptions can be written as some $E_a \in \Ecal(\Hcal_a^-, 1) \cup \Ecal(\Hcal_a^+, 1)$. 
In Definition~\ref{defn:predicable_maximin}, 
we present a novel choice of $\lambda$ that defines a test martingale in $\Ecal(\Hcal_a^+, b)$ and $\Ecal(\Hcal_a^-, b)$ for each arm $a \in [K]$.
We derive the minimax optimal $e$-power over all valid $e$-process in Theorem~\ref{thm:minmax_opt_e_power}, and then show our novel $e$-process achieves this bound in Theorem~\ref{thm:minmax_opt_gen_bern},
justifying it as our sequential test choice.

\subsection{Minimax Optimal Sequential Testing}\label{subsec:minimax}
Our test supermartingales (Definition \ref{defn:predicable_maximin}) are a special case of the universal test supermartingales from Definition \ref{defn:test_martingale}. 

\begin{definition}[Generalized Bernoulli $e$-Process]\label{defn:predicable_maximin}
    For any $b \in [0,1]$, our predictable plugin $e$-processes for $a \in [K]$ is given by:
    $$E_t^\text{PrPl}(\Hcal_a^-, b) = \prod_{i\leq t: A_i = a} (1+ \lambda^-_{t,a} (X_i - \xi)  ),  $$
    $$E_t^\text{PrPl}(\Hcal_a^+, b) = \prod_{i\leq t: A_i = a} (1+ \lambda^+_{t,a}(X_i - \xi) ),  $$
    where $\lambda_{t,a}^-$, $\lambda_{t,a}^+$ are defined as
    \begin{align}
    \lambda^-_{t,a} = \min\left(\frac{b}{\xi},\max\left( \frac{\hat\mu_{t-1}(a) - \xi}{\xi (1-\xi)}, 0 \right)\right)\label{eq:lambda_minus}\\ 
    \lambda^+_{t,a} = \min\left(0,\max\left( \frac{\hat\mu_{t-1}(a) - \xi}{\xi (1-\xi)}, \frac{-b}{(1-\xi)} \right)\right);\label{eq:lambda_plus}
    \end{align}
    $\hat\mu_{t-1}(a) = \frac{\sum_{i \leq t-1: A_i = a} X_i}{N_{t-1}(a)}$; $N_{t-1}(a) = \sum_{i =1}^{t-1} \mathbf{1}[A_i =a]$ is the number of draws from arm $a$ by time $t-1$; and $\hat\mu_{0}(a) = \xi$. 
\end{definition}

Our test supermartingales follow the \emph{predictable plugin} strategy \citep{waudbysmith2022estimating}, where we use $\Fcal_{t-1}$-measurable choice of $\lambda$ in order to adaptively learn the means $\bm\mu$ of the bandit instance $P \in \Pcal(\bm\mu)$. 

We make two important observations on 
our $e$-processes in Definition \ref{defn:predicable_maximin}.
First, when $b=1$, it is the nonparametric generalization  Bernoulli likelihood ratio test. When $P_a$ is a sequence of Bernoulli distributions with fixed mean $\mu_a$, our approach reduces directly to the optimal sequential test \citep{Wald1945SequentialTO} when $\hat\mu_{t-1}(a) = \mu_a$. More generally, our $e$-processes maintain $\delta$-level error control when testing is conducted according to Proposition \ref{prop:delta_error_control}.

Second, our choice of $\lambda$ converges to the $\lambda$ that attains \emph{minimax optimal} $e$-power as $N_t(a)$, the number of times arm $a$ is sampled, goes to infinity. To formalize this, in Theorem \ref{thm:minmax_opt_e_power} (proof in Appendix \ref{app:proof_thm_minmax_opt_e_power}), we characterize the minimax optimal $e$-power for a fixed conditional mean $\mu_a$, representing the best-case $e$-power across all $e$-processes under the worst-case distributional sequence $P \in \Pcal(\mu_a)$. 
Next, in Theorem~\ref{thm:minmax_opt_gen_bern}, we show that our $e$-process achieves near optimal minimax guarantees. 


\begin{theorem}[Minimax Optimal $e$-power]\label{thm:minmax_opt_e_power} Let $\xi \in (0,1)$, 
and $\mu_a \in (0,1)$. Let $E$ be any possible $e$-proceses with respect to composite null $\Hcal^-_a  = \{\Pcal(\mu_a): \mu_a \leq \xi\}$.
Then, for a fixed number of arm pulls $N_t(a) \in \NN$ of arm $a$ with $\mu_a > \xi$, the optimal growth rate of the e-process under the worst-case instance is
\begin{align}\label{eq:minmax_e_power}
    &\inf_{P \in \Pcal(\mu_a)} \sup_{E = (E_t)_{t \in \NN}} \frac{\EE[\log(E_t)]}{N_t(a)} 
    \\ \quad &= \log\frac{1-\mu_a}{1-\xi} + \mu_a \log\frac{\mu_a(1-\xi)}{\xi(1-\mu_a)}.
\end{align}
Likewise, 
let $E$ be any possible $e$-processes with respect to the composite null $\Hcal_a^+$.
Then, Equation \eqref{eq:minmax_e_power} holds as well for any fixed
$N_t(a) \in \NN$ of arm $a$ with $\mu_a \leq \xi$.
\end{theorem}

Next, in Theorem~\ref{thm:minmax_opt_gen_bern} (proof in Appendix~\ref{proof:thm_2_e_power_optimal}) we show that 
our $e$-processes attain the same power asymptotically, with a vanishing term as
$N_t(a)$ grows large. 

\begin{theorem}[Asymptotic Minimax Lower Bound
of Generalized Bernoulli $e$-Process]\label{thm:minmax_opt_gen_bern}
Let $\xi \in (0,1)$, $b \in (0,1)$, and $\mu_a \in (\xi(1-b),b(1-\xi)+\xi)$. Then, for a fixed number of arm pulls $N_t(a) \in \NN$ of arm $a$ with $\mu_a > \xi$, the worst-case $e$-power of our predictable  plugin e-processes in rejecting the null $\Hcal_a^-$, $\inf_{P \in \Pcal(\mu_a)}\frac{1}{N_t(a)} \EE[\log\left( E_t^{PrPl}(\Hcal_a^-, b) \right)]$, is at least
\begin{align}
    \left(\log\frac{1-\mu_a}{1-\xi} + \mu_a \log\frac{\mu_a(1-\xi)}{\xi(1-\mu_a)}\right) - O\left(\sqrt\frac{\log N_t(a)}{{N_t(a)}}\right).\label{eq:gb_lowerbound}
\end{align}
Likewise, under the same conditions, when $\mu_a \leq \xi$, we can show the Equation  \eqref{eq:gb_lowerbound} holds for the lower bound of $\inf_{P \in \Pcal(\mu_a)}\frac{1}{N_t(a)} \EE[\log\left( E_t^{PrPl}(\Hcal_a^+, b) \right)]$.
\end{theorem}

We note that the proof of Theorem \ref{thm:minmax_opt_gen_bern} relies only on the number of pulls of arms $a$, $N_t(a)$, being fixed, allowing for the actions to be generated by any (potentially data-adaptive) policy $\pi$ and stopped at any $\Fcal_t$-measurable stopping time $\tau$. Thus, Theorem \ref{thm:minmax_opt_gen_bern} provides an \emph{anytime-valid} suboptimality bound. 

By imposing the truncation parameter $b \in (0,1)$ in our e-processes, we restrict the range of $\mu_a$ at which the minimax lower bound can be achieved. In Theorem~\ref{thm:minmax_opt_gen_bern}, this range is provided by $\mu_a\in (\xi(1-b),b(1-\xi)+\xi)$, and
it can be increased to $(0,1)$ by setting $b \in (0,1)$ near $1$. The restriction of $b \not\in \{0,1\}$ avoids scenarios where $E_t^{PrPl}(\Hcal_a^+, b), E_t^{PrPl}(\Hcal_a^+, b)$ are constant across $t \in \NN$. Setting $b=0$ implies our test martingales are equal to $1$ for all $t \in \NN$. We exclude $b=1$ due to the fact that if $X_t = 1$ and $\lambda_{t,a}^- = 1/\xi$ for some $t \in \NN$, $E_{t'}^{PrPl}(\Hcal_a^-, 1) = 0$ for all $t' > t$. This implies that we can never reject the null hypothesis $\Hcal_a^-$ after time $t$. The same logic holds in the case of $X_t = 0$, $\lambda_{t,a}^+ = -\frac{1}{1-\xi}$, and $E_{t'}^{PrPl}(\Hcal_a^+, 1)$. 

Theorem \ref{thm:minmax_opt_gen_bern} states that the worst-case $e$-power of the generalized Bernoulli likelihood ratio is at least the minimax lower bound in Theorem \ref{thm:minmax_opt_e_power} with an additional negative term that vanishes at the rate $O\left(\sqrt{\frac{\log N_t(a)}{N_t(a)}}\right)$. As $N_t(a) \rightarrow \infty$, this additional term vanishes to $0$, giving us the minimax optimality result for our proposed $e$-process. Theorems \ref{thm:minmax_opt_e_power} and \ref{thm:minmax_opt_gen_bern} motivate our use of the generalized Bernoulli likelihood ratio: as $N_t(a)$ grows large, no other $e$-process grows its average value faster under the worst-case distribution sequence $P$. 


\section{Sampling Rules for Good Arm Identification}
\label{sec_5:sampling_scheme}
In this section, we pair regret-optimal sampling schemes with our sequential tests to obtain asymptotically optimal stopping times for GAI. We first provide a high-level condition that is both sufficient and necessary for horizon-free regret-minimizing algorithms to attain the optimal regret rate, and provide examples in the literature that satisfy this condition. 
We describe our algorithm in Algorithm~\ref{alg:i-GAI}, which accommodates earlier stopping based on the number of good arms desired.
Finally, we provide its key theoretical properties.

\subsection{Regret-Minimizing Sampling Schemes}
We impose the following condition on our sampling policy, which is satisfied by a wide class of optimal sampling policies for reward-maximizing approaches.
\begin{definition}[Horizon-Free Regret-Minimizing Policy]\label{defn:regret_minimizing_scheme}
    Let $a^* = \argmax_{a \in [K]}\mu_a$. We say that a sampling policy $\pi_t: \Fcal_{t-1} \rightarrow \Delta^K$
    is a \emph{horizon-free regret-minimizing policy} if (i) the algorithm does not take a horizon parameter $t$ as input\footnote{We note that other works \citep{moss_anytime} denote this property as anytime. We opt not to use this name due to the potential confusion with anytime-valid testing, which is also discussed in this work.} and (ii) for any $t \in \NN$, there exists $c \in \RR^+$ independent of $t$ such that $\PP_{\pi}(A_t \neq a^*) \leq c/t$.
\end{definition}

For any horizon-free reward maximization algorithm, the condition in Definition \ref{defn:regret_minimizing_scheme} is both necessary and sufficient to achieve the optimal regret bound of order $O(\log(t))$ (defined with respect to arm mean differences, as in \citealt{Lai_Robbins_1985}). 

An example of a regret-minimizing policy that achieves this rate 
is MOSS-anytime \citep{moss_anytime}. For our empirical results, we use a modified version of MOSS-anytime (Algorithm \ref{alg:modified_moss}) to accommodate sequential elimination of arms once they have been labeled as good or bad. Our restriction for selecting arms in $\Ical_{t-1}$ ensures that once an arm has been labeled as good or bad, it is not sampled again in future timesteps. Given a regret-minimizing sampling strategy $\pi = (\pi_t)_{t \in \NN}$ that satisfies Definition \ref{defn:regret_minimizing_scheme}, we now introduce our GAI algorithm, and provide theoretical results that show our approach asymptotically attains the minimax lower bounds for GAI stopping times. 

\begin{algorithm}[t] 
  \caption{Modified MOSS-anytime Policy $\pi_t$.} 
  \label{alg:modified_moss}
  \begin{algorithmic}[1] 
  \State \textbf{input}: $\alpha > 0$, observed samples $(A_i, X_i)_{i < t}$, unlabeled set $\Ical_{t-1}$.

  \State Find the arm $a$ with the largest upper confidence bound among the unlabeled arms in $\Ical_{t-1}$:
  $$ A_t = \argmax_{a \in \Ical_{t-1}}\hat\mu_{t-1}(a) + \sqrt{\frac{1+\alpha}{2} \frac{\max(0, \log\frac{t}{K N_{t-1}(a)})}{N_{t-1}(a)}} $$
  with $N_{t-1}(a)$, $\hat\mu_{t-1}(a)$ as defined in Definition \ref{defn:predicable_maximin}.
  
  \State \textbf{output}: $A_t$.
  \end{algorithmic}    
\end{algorithm}




\subsection{A Minimax Optimal Approach for GAI}

We describe
our approach
in Algorithm \ref{alg:i-GAI}, which provides a \emph{minimax optimal} solution for the GAI problem. At each round, Algorithm \ref{alg:i-GAI} samples an arm according to the modified version of MOSS provided in Algorithm \ref{alg:modified_moss}, and uses the $e$-processes $E_t^{\text{PrPl}}(\Hcal_{A_t}^-, b)$ and $E_t^{\text{PrPl}}(\Hcal_{A_t}^-, b)$ to test whether we can label the arm as good or bad. This test is conducted in the manner of Proposition \ref{prop:delta_error_control}. 
If the $e$-process $E_t^{\text{PrPl}}(\Hcal_{A_t}^-, b)$ surpasses the threshold $2K/\delta$, 
we reject the null hypotheses $\Hcal_{A_t}^- = \{\Pcal(\mu_a): \mu_a \leq \xi\}$, meaning that we deem the arm as good. 
Likewise, if the test statistic $E_t^{\text{PrPl}}(\Hcal_{A_t}^+, b)$ surpasses the threshold $2K/\delta$, 
we reject the null hypotheses $\Hcal_{A_t}^+ = \{\Pcal(\mu_a): \mu_a > \xi\}$, labeling the arm as bad. 

We note that the threshold for the $e$-processes is $2K/\delta$, not $1/\delta$, in order to provide $\delta$-correct error guarantees. This threshold results from a simple union-bound over two tests (one for labeling arms as good, one for labeling arms as bad) for each of the $K$ arms. While different constructions of $e$-processes, such as the averaging method done in \cite{cho2024peeking}, provide strictly larger values of the $e$-processes almost surely, our union bounds suffice to attain asymptotically optimal results. 

In Algorithm \ref{alg:i-GAI}, we also include the optional parameter of $m$, which specifies the number of desired good arms and permits early stopping before all arms are labeled. When $m = K$, our algorithm solves GAI problem, where we stop at $\tau_{\text{stop}}$, the time in which we label all arms as good or bad. When $m < G_{\bm\mu}$, i.e., the desired number of good arms is less than the unknown number of good arms, we terminate our GAI algorithm as soon as we identify $m$ good arms, i.e., at time $\tau_{\Gcal, m}$. If $m$ good arms are not identified by our algorithm before labeling all arms, we still terminate at $\tau_{\text{stop}}$, ensuring that misspecification of $m$ does not harm performance relative to solving the full GAI problem.   

\begin{algorithm}[t] 
  \caption{Minimax Optimal GAI} 
  \begin{algorithmic}[1] 
  \State \textbf{input}: sampling policy $\pi$ for GAI (e.g., Alg. \ref{alg:modified_moss}), error parameter $\delta$, truncation constant $b$, desired number of good arms $m$ (optional). 
  \State Set good arm and bad arm set $\Gcal_0, \Bcal_0$ as empty, and unlabeled arm set $\Ical_0 = [K]$. 
  \State Sample each arm once, and 
  initialize counter $t = K$.
  \While{$|\Gcal| < m$ and $\Ical_t \neq \emptyset$}
      \State Set $t = t+1$. 
      \State Select an arm $A_t \in \Ical_{t-1}$ based on policy $\pi$. 
      \State Calculate $E_t^{\text{PrPl}}(\Hcal_{A_t}^-, b), E_t^{\text{PrPl}}(\Hcal_{A_t}^+, b)$ as provided in Definition \ref{defn:predicable_maximin}. 
      \If{$E_t^{\text{PrPl}}(\Hcal_{A_t}^-, b) \geq \frac{2K}{\delta}$}{
         set $\Ical_t = \Ical_{t-1} \setminus A_t$, and $\Gcal_t = \Gcal_{t-1} \cup A_t$.} 
      \EndIf
    
      \If{$E_t^{\text{PrPl}}(\Hcal_{A_t}^+, b) > \frac{2K}{\delta}$} {set $\Ical_t = \Ical_{t-1} \setminus A_t$, and $\Bcal_t = \Bcal_{t-1} \cup A_t$.}
      \EndIf
  \EndWhile
  \State \textbf{output}: $\Gcal_t, \Bcal_t, \tau = t$.
  \end{algorithmic}
  \label{alg:i-GAI}
\end{algorithm}

\subsection{Theoretical Guarantees}

We first provide our error guarantees in Theorem \ref{thm:error_control_alg_2}, which ensures our approach satisfies the $\delta$-error level constraint defined in Definition \ref{defn:error_control}. 

\begin{theorem}[$\delta$-level Error Control for Algorithm \ref{alg:i-GAI}]\label{thm:error_control_alg_2}
For any $P \in \Pcal(\bm\mu)$, for any $\delta \in [0,1]$, $b \in [0,1]$, $m \in [K]$, Algorithm \ref{alg:i-GAI} satisfies Definition \ref{defn:error_control}. 
\end{theorem}
Theorem \ref{thm:error_control_alg_2} (proof in Appendix~\ref{proof:error_control_alg_2}) states that for any choice of input parameters, we ensure $\delta$-level error control across the duration of the algorithm. While $\delta$-level error control ensures the desired correctness of our outputs, they do not ensure that our algorithm performs well in terms of stopping times $\tau_{G,i}$ for all $i\in \{1,...,G_{\bm\mu}\}$ and $\tau_{\text{stop}}$. To contextualize our stopping time results, we provide asymptotic minimax lower bounds on the stopping times for any GAI approach with $\delta$-level error control, and show that our approach achieves this lower bound.

\begin{theorem}[Asymptotic Minimax Optimality]\label{theorem:minimax_lower_bounds_stopping_time}
    Let (i) $\xi \in (0,1)$, $b \in (0,1)$, (ii) $\bm\mu \in \left[\xi(1-b), b(1-\xi) + \xi\right]^K \setminus [\xi]^K$, and (iii) $\mu_i \neq \mu_j$ for any $i\neq j$. Define $d(\mu_a, \xi) = \log\frac{1-\mu_a}{1-\xi} + \mu_a\log\frac{\mu_a(1-\xi)}{\xi(1-\mu_a)}$. For any $\delta$-correct fixed confidence GAI algorithm $(\pi, \tau)$, the asymptotic minimax stopping time $\tau_{\text{stop}}$ is lower bounded by:
    \begin{align}
        \liminf_{\delta \rightarrow 0 \ (\pi, \tau)}\sup_{P \in \Pcal(\bm\mu)}\frac{\EE[\tau_{\text{stop}}]}{\log(1/\delta)} \geq \sum_{a=1}^K \frac{1}{d(\mu_a, \xi)}.
    \end{align}
    Consider the following modifications for Algorithm \ref{alg:i-GAI}. In lines 8-9 of Algorithm \ref{alg:i-GAI}, we reset $E_t(\Hcal_a^-, b) = E_t(\Hcal_a^+, b) = 1$,  $\hat\mu_{t-1}(a) = \xi$, $N_{t}(a) = 0$ for all $a \in \Ical_t$. For $m \in [\Gcal_{\bm\mu},K]$, Algorithm \ref{alg:i-GAI} achieves the minimax lower bound for $\tau_{\text{stop}}$, i.e.,  
    \begin{equation}
        \lim_{\delta \rightarrow 0}\sup_{P \in \Pcal(\bm\mu)} \frac{\EE[\tau_{\text{stop}}]}{\log(1/\delta)} \leq \sum_{a=1}^K \frac{1}{d(\mu_a, \xi)}.
    \end{equation} 
    Furthermore, let $\mu_{i^*}$ denote the $i$-th largest mean among $\bm{\mu}$. For all $i \leq G_{\bm\mu}$, the asymptotic minimax lower bound on the stopping time $\tau_{\Gcal, i}$ is given by:
    \begin{align}\label{eq:min_max_good_bound}
        \liminf_{\delta \rightarrow 0 \ (\pi, \tau)}\sup_{P \in \Pcal(\bm\mu)} \frac{\EE[\tau_{\Gcal, i}]}{\log(1/\delta)} \geq 
        \quad \sum_{a: \mu_a \geq \mu_{i^*}} \frac{1}{d(\mu_a, \xi)}.
    \end{align}
     For all $m \in [K]$, $i \leq \min(m, \Gcal_{\bm\mu})$, our modified version of Algorithm \ref{alg:i-GAI} achieves the minimax lower bound for $\tau_{\Gcal, i}$, i.e.,
    \begin{equation}\label{eq:min_max_good}
        \lim_{\delta \rightarrow 0}\sup_{P \in \Pcal(\bm\mu)} \frac{\EE[\tau_{\Gcal, i}]}{\log(1/\delta)} \leq \sum_{a:\mu_a \geq \mu_{i^*}}\frac{1}{d(\mu_a, \xi)}.
    \end{equation}
\end{theorem}

As error control becomes stricter (i.e., $\delta \rightarrow 0$), Theorem \ref{theorem:minimax_lower_bounds_stopping_time} guarantees that the largest expected stopping times of Algorithm \ref{alg:i-GAI} are at most the smallest expected stopping times among any $\delta$-correct GAI algorithm in the worst-case setting. When the desired number of good arms is misspecified (i.e., $m \geq \Gcal_{\mu}$) or we want all arms to be labeled ($m = K$), the expectations of stopping times $(\tau_{\Gcal, i})_{i \leq \Gcal_{\bm\mu}}$ and $\tau_{\text{stop}}$ of our approach are 
no worse than their respective minimax lower bounds. When $m \leq G_{\bm\mu}$, Equations \eqref{eq:min_max_good_bound} and \eqref{eq:min_max_good} imply we expect to stop at the minimax optimal time for identifying $m$ good arms. 

Our modifications for Algorithm \ref{alg:i-GAI} effectively restarts the sampling scheme $\pi$ and testing procedures as if we had collected no information up to time $\tau_{\Gcal, 1}$. We emphasize that this is for analytical convenience for analyzing the limiting stopping times $\tau_{\Gcal, i}$ and $\tau_{\text{stop}}$. In practice, discarding such information is likely to cause far worse performance than the empirical results using Algorithm \ref{alg:i-GAI} in the main body of the paper.


\paragraph{When does reward maximization align with GAI?} We highlight the special case where there exists at least one good arm, and $m=1$. This setting corresponds to scenarios where we desire a treatment with satisfactory effect as quickly as possible, and has applications in system verification \citep{degenne2023existence} and financial portfolio risk \citep{juneja2019samplecomplexitypartitionidentification}. For this problem of finding one good arm, the proof of Theorem \ref{theorem:minimax_lower_bounds_stopping_time} in Appendix \ref{proof:thm_stopping_times} implies that MOSS-anytime is an optimal sampling strategy for achieving minimax optimal stopping time $\tau_{\Gcal, 1}$ when tight error control (i.e., small $\delta$) is desired. This emphasizes that in many practical applications, the pure exploration problem does not require a trade off with regret minimization: experimenters can obtain the fastest time for obtaining a $\delta$-correct conclusion, while study subjects obtain the best possible outcomes at an optimal rate.

\section{Empirical Results}
\label{sec_7:empirical_results}
\begin{table*}[tbh]
    \centering
    \fontsize{8}{7}\selectfont
    \begin{tabular}{lll|rrrrrrrr}
$K$ & Dist. & & \multicolumn{2}{c}{$\tau_{\Gcal, 1}$} & \multicolumn{2}{c}{$\tau_{\Gcal, 2}$} & \multicolumn{2}{c}{$\tau_{\text{stop}}$} & \multicolumn{2}{c}{$R[\tau_{\Gcal, 1}]$} \\
\midrule
\multirow[c]{12}{*}{4} & \multirow[c]{6}{*}{Bern} & OPT & 255.1 $\pm$ & 177.5 & 1231.3 $\pm$ & 595.5 & 2505.6 $\pm$ & 896.5 & 0.0 $\pm$ & 0.0 \\
 &  & \textbf{Alg. 2} & \textbf{532.8} $\pm$ & \textbf{361.6} & \textbf{1954.8} $\pm$ & \textbf{892.1} & \textbf{3588.6} $\pm$ & \textbf{1213.4} & \textbf{19.0} $\pm$ & \textbf{17.1} \\
 &  & HDoC & 1422.3 $\pm$ & 552.5 & 5692.7 $\pm$ & 1556.7 & 10729.0 $\pm$ & 2125.1 & 40.0 $\pm$ & 18.8 \\
 &  & LUCB-G & 1553.0 $\pm$ & 686.3 & 5999.9 $\pm$ & 1673.9 & 10743.0 $\pm$ & 2137.1 & 54.2 $\pm$ & 25.5 \\
 &  & APT-G & 9429.5 $\pm$ & 1835.1 & 10665.4 $\pm$ & 2049.5 & 10739.6 $\pm$ & 2027.8 & 930.5 $\pm$ & 198.6 \\
\cmidrule(lr){2-11}
 & \multirow[c]{6}{*}{Mix} & OPT & 190.7 $\pm$ & 85.1 & 939.5 $\pm$ & 341.0 & 1963.4 $\pm$ & 528.0 & 0.0 $\pm$ & 0.0 \\
 &  & \textbf{Alg. 2} & \textbf{355.1} $\pm$ & \textbf{201.1} & \textbf{1279.0} $\pm$ & \textbf{580.4} & \textbf{2366.4} $\pm$ & \textbf{837.0} & \textbf{13.6} $\pm$ & \textbf{9.9} \\
 &  & HDoC & 1407.5 $\pm$ & 458.0 & 5629.3 $\pm$ & 1286.4 & 10856.6 $\pm$ & 1845.3 & 38.9 $\pm$ & 15.5 \\
 &  & LUCB-G & 1539.0 $\pm$ & 452.8 & 5904.9 $\pm$ & 1329.0 & 10927.4 $\pm$ & 1908.7 & 54.5 $\pm$ & 17.8 \\
 &  & APT-G & 9604.7 $\pm$ & 1558.1 & 10871.0 $\pm$ & 1707.9 & 10936.0 $\pm$ & 1712.5 & 969.4 $\pm$ & 177.6 \\
\cmidrule(lr){1-11}
\multirow[c]{12}{*}{10} & \multirow[c]{6}{*}{Bern} & OPT & 290.6 $\pm$ & 169.8 & 1500.3 $\pm$ & 732.6 & 7578.7 $\pm$ & 1685.0 & 0.0 $\pm$ & 0.0 \\
 &  & \textbf{Alg. 2} & \textbf{827.4} $\pm$ & \textbf{579.3} & \textbf{2596.9} $\pm$ & \textbf{1006.4} & \textbf{10319.4} $\pm$ & \textbf{2055.9} & \textbf{50.5} $\pm$ & \textbf{39.5} \\
 &  & HDoC & 1898.3 $\pm$ & 635.2 & 7048.2 $\pm$ & 1664.1 & 27999.0 $\pm$ & 3464.8 & 112.1 $\pm$ & 36.7 \\
 &  & LUCB-G & 2306.6 $\pm$ & 792.8 & 7845.2 $\pm$ & 1878.8 & 27943.5 $\pm$ & 3446.6 & 168.1 $\pm$ & 50.9 \\
 &  & APT-G & 24674.5 $\pm$ & 3243.1 & 27721.1 $\pm$ & 3533.1 & 28054.5 $\pm$ & 3525.1 & 3348.8 $\pm$ & 427.0 \\
\cmidrule(lr){2-11}
 & \multirow[c]{6}{*}{Mix} & OPT & 210.3 $\pm$ & 80.7 & 1099.3 $\pm$ & 354.6 & 5704.8 $\pm$ & 860.6 & 0.0 $\pm$ & 0.0 \\
 &  & \textbf{Alg. 2} & \textbf{513.6} $\pm$ & \textbf{293.6} & \textbf{1756.9} $\pm$ & \textbf{680.8} & \textbf{6921.3} $\pm$ & \textbf{1239.2} & \textbf{33.1} $\pm$ & \textbf{20.8} \\
 &  & HDoC & 1894.0 $\pm$ & 599.8 & 7033.4 $\pm$ & 1454.9 & 28452.0 $\pm$ & 2805.4 & 110.7 $\pm$ & 33.2 \\
 &  & LUCB-G & 2333.9 $\pm$ & 640.8 & 8037.3 $\pm$ & 1596.4 & 28394.8 $\pm$ & 2767.7 & 169.9 $\pm$ & 43.4 \\
 &  & APT-G & 25316.0 $\pm$ & 2409.9 & 28430.8 $\pm$ & 2739.5 & 28666.2 $\pm$ & 2709.7 & 3433.8 $\pm$ & 330.1 \\
\cmidrule(lr){1-11}
\multirow[c]{12}{*}{20} & \multirow[c]{6}{*}{Bern} & OPT & 317.2 $\pm$ & 154.4 & 1679.2 $\pm$ & 700.9 & 16870.0 $\pm$ & 2480.8 & 0.0 $\pm$ & 0.0 \\
 &  & \textbf{Alg. 2} & \textbf{1085.8} $\pm$ & \textbf{728.1} & \textbf{3480.6} $\pm$ & \textbf{1758.1} & \textbf{22417.2} $\pm$ & \textbf{2864.7} & \textbf{85.4} $\pm$ & \textbf{61.2} \\
 &  & HDoC & 2606.2 $\pm$ & 862.2 & 9263.8 $\pm$ & 2366.6 & 57720.3 $\pm$ & 4800.7 & 232.6 $\pm$ & 70.2 \\
 &  & LUCB-G & 3551.1 $\pm$ & 1085.9 & 11373.0 $\pm$ & 2861.6 & 57864.9 $\pm$ & 4727.6 & 366.6 $\pm$ & 101.1 \\
 &  & APT-G & 51586.8 $\pm$ & 4899.0 & 57016.2 $\pm$ & 4960.8 & 57865.4 $\pm$ & 4874.9 & 7620.7 $\pm$ & 702.9 \\
\cmidrule(lr){2-11}
 & \multirow[c]{6}{*}{Mix} & OPT & 245.7 $\pm$ & 85.2 & 1238.2 $\pm$ & 406.4 & 12754.6 $\pm$ & 1274.8 & 0.0 $\pm$ & 0.0 \\
 &  & \textbf{Alg. 2} & \textbf{678.1} $\pm$ & \textbf{409.7} & \textbf{2219.0} $\pm$ & \textbf{819.1} & \textbf{15123.7} $\pm$ & \textbf{2066.8} & \textbf{58.1} $\pm$ & \textbf{36.2} \\
 &  & HDoC & 2674.0 $\pm$ & 687.7 & 9413.8 $\pm$ & 1687.0 & 58446.1 $\pm$ & 3700.6 & 238.8 $\pm$ & 56.5 \\
 &  & LUCB-G & 3376.2 $\pm$ & 949.0 & 11566.3 $\pm$ & 1998.0 & 58370.7 $\pm$ & 3719.3 & 354.2 $\pm$ & 89.3 \\
 &  & APT-G & 52153.7 $\pm$ & 3998.9 & 58004.5 $\pm$ & 4075.6 & 58549.6 $\pm$ & 4019.2 & 7709.9 $\pm$ & 560.2 \\
 \midrule
  \multicolumn{2}{c}{\multirow[c]{5}{*}{Dose-finding}} &OPT & 2046.3 $\pm$ & 1210.1 & \multicolumn{2}{c}{\text{-}} & 7288.8 $\pm$& 2607.6 & 0.0 $\pm$ & 0.0 \\
 &  & \textbf{Alg. 2} & \textbf{3444.7} $\pm$ & \textbf{1665.5} & \multicolumn{2}{c}{\text{-}  }& \textbf{10587.9} $\pm$ & \textbf{3291.6} & \textbf{41.4} $\pm$& \textbf{23.5} \\
 &  & HDoC & 9797.5 $\pm$& 3047.7 & \multicolumn{2}{c}{\text{-}  } & 31726.1 $\pm$& 5617.9 & 108.9 $\pm$& 36.4 \\
 &  & LUCB-G & 10333.8 $\pm$& 3270.3 & \multicolumn{2}{c}{\text{-}  } & 31779.4 $\pm$& 5665.6 & 169.4 $\pm$& 47.5 \\
 &  & APT-G & 31087.9 $\pm$& 5598.5 & \multicolumn{2}{c}{\text{-}  } & 31807.3 $\pm$& 5708.6 & 1662.9 $\pm$& 330.6 \\
\cmidrule(lr){1-11}
\end{tabular}
    \caption{Average stopping times and standard deviations for 200 independent runs. We bold the lowest average stopping times outside of OPT. $\tau_{\Gcal, 2}$ is omitted for Dose-finding (no runs find two good arms).}
    \label{tab:simulation-results} 
\end{table*}

In this section, we provide empirical results for Algorithm \ref{alg:i-GAI}, showing our method performs well beyond the minimax setting and outperforms all existing approaches for the GAI problem. For all simulations, we use error tolerance $\delta = 0.05$, $m = K$, and 200 simulations for each setting of arm distribution and $K$.

\textbf{Simulation Settings.} For all simulations, we have 2 good arms, and set the threshold to $\xi = 0.5$. We vary the total number of arms $K \in \{4, 10, 20\}$ to test different levels of good arm sparsity. For each $K$, we test both Bernoulli distribution arms, reflecting the worst-case instance, and a mixture distribution between Bernoulli and Uniform distributions, keeping the mean vector $\bm\mu \in \RR^K$ constant for each $K$. For all $K$, arms have means within $\xi \pm 0.1$ to ensure our problem is appropriately difficult (i.e., hard to distinguish whether a mean is above/below the threshold value $\xi = 0.5$), reflecting common use-cases in practice. We refer to Appendix \ref{app:exp_details} for full details on arm distributions. 

\textbf{Case Study.} We simulate the dose-finding experiment in \cite{kano2019good}, with $\xi = 0.5$ and Bernoulli arms with means $\bm\mu = [0.36, 0.34, 0.469, 0.465, 0.537]$. The means represent placebo, and secukinumab 25mg, 75mg, 150mg, and 300mg, respectively. The expected reward indicates American College of Rheumatology 20\% response (ACR) at week 16, provided in Table 2 of \cite{Genovese863}.


\textbf{Baselines and Metrics.} For our baselines, we use three different GAI algorithms, LUCB-G, APT-G, and HDoC, 
included in \cite{kano2019good}. We also include an instance-wise minimax optimal GAI algorithm that uses oracle knowledge of $\bm\mu$, denoted as OPT. We provide further detail on the sampling schemes and stopping times in Appendix \ref{app:exp_details}. For our metrics of comparison, we keep track of the following times: $\tau_{\Gcal, 1}$, $\tau_{\Gcal, 2}$, and $\tau_{\text{stop}}$, the first times in which we label one good arm, two good arms, and all arms, respectively. To show that our approach results in small regret when identifying one good arm, we track $R[\tau_{\Gcal, 1}] = \EE[\tau_{\Gcal, 1}\max_{a \in [K]}\mu_a - \sum_{t = 1}^{\tau_{\Gcal, 1}}X_t]$, the regret incurred up to time $\tau_{\Gcal, 1}$. For all methods shown, no arms were mislabeled across all simulations. 

\paragraph{Discussion of Results} Our experiments show that Algorithm \ref{alg:i-GAI} performs well across both distribution types, and for any level of good arm sparsity. Across all simulation settings, our approach provides smaller stopping times relative to all methods other than OPT. The differences between average stopping times for our method and OPT are smaller than those between our method and the next best method. This shows that our approach achieves results in roughly the same order as OPT relative to other methods.

The Bernoulli distributions provide an example of the minimax setting, where our approach achieves asymptotic optimality. In this setting, our approach provides at least a 60\% reduction across all the expected stopping times relative to any baseline without oracle knowledge. We see a similar result with our mixture distribution, where we obtain roughly a 75\% reduction in the average stopping times $\tau_{\Gcal, 1}$, $\tau_{\Gcal, 2}$, and $\tau_{\text{stop}}$. This shows our approach has strong empirical performance well beyond the minimax case. With higher degrees of sparsity (e.g., larger $K$), all algorithms suffer larger stopping times for identifying good arms ($\tau_{\Gcal, 1}, \tau_{\Gcal, 2}$) relative to OPT. While more arms necessarily lead to more exploration costs, the reductions in expected stopping time from the next best non-oracle method remain similar across all $K$, showing our approach maintains its relative performance even in sparser settings. We use $\tau_{\text{stop}}$ as a proxy for the improved power of our test. The average value of $\tau_{\text{stop}}$ serves as a proxy for the power of our sequential test, as each arm must be sufficiently sampled to be labeled. Across all distributions, we see that our stopping rules drastically reduce $\tau_{\text{stop}}$. This shows our novel sequential test significantly improves detection power, even beyond the minimax optimal $e$-power results of Theorem \ref{thm:minmax_opt_gen_bern}. 


We demonstrate the value of this minimum-harm approach in our semi-synthetic case study. Here, there only exists one good arm, and our approach identifies this arm using 65\% less samples on average than any other non-oracle method. Using reward maximizing schemes over the duration of $\tau_{\Gcal, 1}$, we ensure that we minimize regret (i.e., welfare loss due to receiving sub-optimal treatments) for this short duration, resulting in a 60\% reduction of $R[\tau_{\Gcal, 1}]$ relative to the best non-oracle methods. This demonstrates a case where the pure exploration problem of interest directly aligns with reward maximization.

\section{Conclusions and Future Directions}
\label{sec_8:conclusion}

In this work, we provide an approach for GAI that (i) ensures $\delta$-level error control under minimal nonparametric assumptions on arm distributions, (ii) asymptotically achieves the minimax optimal stopping times for labeling any number of good arms, regardless of the problem instance, and (iii) aligns directly with regret-optimal reward maximization sampling schemes in the case where the experimenter seeks to find one good arm as quickly as possible. Future directions include second-order minimax optimality, and exploring different pure exploration problems.


\newpage
\bibliography{citation}

\appendix

\thispagestyle{empty}

\onecolumn 
\section*{Appendix}
\section{Notation} \label{app:notation}
\begin{description}[
  leftmargin=2.2cm,
  labelwidth=2cm,
  labelsep=.2cm,
  parsep=0mm,
  itemsep=2pt
  ]
\item[${[}K{]}$] set of integers $1,\ldots, K$, where $K$ is the total number of arms
\item[$\Delta^K$] probability simplex over the K arms
\item[$\pi_t$] the sampling policy at time $t$ that takes in the history up to the current time and decides the probability of sampling each arm; $\pi_t: (X_i,A_i)_{i=1}^{t-1}\rightarrow \Delta^K$
\item[$\Fcal_t$] the canonical filtration at time $t$; $\Fcal_t =\sigma ((A_i, X_i)_{i=1}^t)$
\item[$\Fcal_0$] the empty sigma field
\item[$\mu_a$] the mean of arm $a$ at time $t$ when conditioned on $\Fcal_{t-1}$
\item[$\bm\mu$] the vector containing $K$ conditional arm means; $\bm\mu \equiv \{\mu_1, \cdots, \mu_K\}$
\item[$\Pcal(\mu_a)$] the set of all \emph{distributional sequences} on $[0,1]^\infty$ such that $\EE[X_t| A_t=a, \Fcal_{t-1}] = \mu_a$
\item[$\Pcal(\bm\mu)$] the set of all \emph{distributional sequences} on $[0,1]^\infty$ such that $\EE[X_t| A_t=a, \Fcal_{t-1}] = \mu_a \; \forall a\in [K]$; $\Pcal(\bm\mu) \equiv \cap_{a \in k} \Pcal(\mu_a)$
\item[$\xi$] threshold value
\item[$\Gcal_{\bm\mu}$] the set of \emph{true} good arms; $ \Gcal_{\bm\mu}= \{a \in [K]: \mu_a  > \xi\}$
\item[$\Bcal_{\bm\mu}$] the set of \emph{true} bad arms; $\Bcal_{\bm\mu} = [K]\setminus \Gcal_{\bm\mu}$
\item[$\tau_a$] the first time arm $a$ is labeled as either good or bad 
\item[$\tau_{\Gcal, i}$] the first time $i$ arms are labeled as good; $\tau_{\Gcal, i} = \inf\{t \in \NN: |\Gcal_t| = i\}$
\item[$\tau_{\text{stop}}$] the first time all arms are labeled good or bad; $\tau_{\text{stop}} = \inf\{t \in \NN: \Gcal_{t} \cup \Bcal_{t} = [K]\}$
\item[$\delta$] a fixed error rate
\item[$\Hcal^-_a$] the composite hypothesis that the conditional mean of arm $a$ is \emph{below or equal} to the threshold; $\Hcal^-_a=\{\Pcal(\mu_a): \mu_a \leq \xi\}$
\item[$\Hcal^+_a$] the composite hypothesis that the conditional mean of arm $a$ is \emph{above} the threshold; $\Hcal^+_a= \{\Pcal(\mu_a):\mu_a > \xi\}$
\item[$E$] e-process; a sequence of $\Fcal_t$-measurable random variables $(E_t)_{t \in \NN}$ satisfying 1) $E$ is nonnegative and (2) $\sup_{P \in \Hcal}\sup_{\tau}\EE_P[E_\tau] \leq 1$, where first supremum is taken over all possible distributions in the null set $\Hcal$, and the second supremum is taken over all (potentially infinite) stopping times $\tau$
\item[$T_t(E, \delta)$] anytime-valid test representing $\mathbf{1}[E_t\geq 1/\delta]$, where a value of $1$ indicates rejection of the null hypothesis at time $t$
\item[$\EE_{P}(\log(E_t))/t$] 
$e$-power of any hypothesis $P$ not in the null hypothesis set $\Hcal$ of an $e$-process $E$; the logarithmic growth rate of an $e$-process normalized by the number of observations $t$ 
\item[$\lambda$] a $\Fcal_{t-1}$-measurable univariate parameter
\item[$E_a$]  an $e$-process parametrized by $\lambda$, where the sequence of random variables $\left(E_t^a(\lambda, \xi)\right)_{t \in \NN}$ is characterized by $E_t^a(\lambda, \xi) = \prod_{i: A_i = a} (1+  \lambda(X_i - \xi) )$
\item[$b$] truncation constant in $[0,1]$ ensuring that our test martingales are strictly non-negative and are $e$-processes
\item[$\Ecal(\Hcal^+_a, b)$] class of $e$-processes $E_a$ where the range of $\lambda$ is in $ [-b/(1-\xi), 0]$
\item[$\Ecal(\Hcal^-_a, b)$] class of $e$-processes $E_a$ where the range of $\lambda$ is in $ [0, b/\xi]$
\item[$N_{t}(a)$] the number of draws from arm $a$ by time $t$
\item[$\hat\mu_{t}(a)$] the empirical average of the mean of arm $a$ at time $t$; $\hat\mu_{t}(a) = \frac{\sum_{i \leq t: A_i = a} X_i}{N_{t}(a)}$
\item[$(E_t^\text{PrPl}(\Hcal_a^-, b))_{t\in\mathbb{N}}, (E_t^\text{PrPl}(\Hcal_a^+, b))_{t\in\mathbb{N}}$] Our proposed predictable plugin $e$-processes as defined in Definition~\ref{defn:predicable_maximin}
\item[$a^*$] the arm with the highest expected conditional mean;  $a^* = \argmax_{a \in [K]}\mu_a$
\item[$c$] a positive constant independent of $t$ such that the probability of not selecting arm $a^*$ under policy $\pi$ is upper bounded by $c/t$; $\PP_{\pi}(A_t \neq a^*) \leq c/t$
\item[$\Ical_{t-1}$] a set of unlabeled arms returned by the algorithm by the end of the time step $t-1$
\item[$m$] optional parameter denoting the number of good arms $m$ that enable early stopping of GAI when $|\Gcal_{\bm\mu}|\geq m$
\item[$\Gcal_t$] good arm set returned by the algorithm at time $t$
\item[$\Bcal_t$] bad arm set returned by the algorithm at time $t$
\item[$d(\mu_a, \xi)$] The KL divergence between 2 Bernoulli distributions with means $\mu_a$ and $\xi$; $\log\frac{1-\mu_a}{1-\xi} + \mu_a\log\frac{\mu_a(1-\xi)}{\xi(1-\mu_a)}$
\item[$\mu_{i^*}$] the $i$-th largest mean among $\bm\mu$
\item[$R(\tau_{\Gcal, 1})$] regret incurred up to time $\tau_{\Gcal, 1}$; $R[\tau_{\Gcal, 1}] = \EE[\tau_{\Gcal, 1}\max_{a \in [K]}\mu_a - \sum_{t = 1}^{\tau_{\Gcal, 1}}X_t]$
\end{description}

\section{Experiment Details} \label{app:exp_details}
In this section, we provide details about our experimental setup and two ablation studies to test the performance of different combinations of policies and stopping times.

\subsection{Simulation details}
\subsubsection{Reward Data-Generating Process (DGP)}\label{subsec:dgp}
We tested two reward DGPs for each arm: Bernoulli (Bern) and Mixture (Mix). The mixture DGP is generated by averaging one Bernoulli distribution and one Uniform (Unif) distribution:
\[p_k(x) = \frac{p_{\text{Bern}(2\mu_k - 1/2)}(x)}{2} + \frac{p_{\text{Unif(0, 1)}}(x)}{2}.
\]
Note that \[
\mu_\text{Mix} = \frac{\mu_\text{Bern} + \mu_\text{Unif}}{2} = \frac{(2\mu_k - 1/2) + (1/2)}{2} = \mu_k.
\] The two DGPs share the same mean for each arm. 

\subsubsection{Number of arms and mean vectors}
We set $\xi = 0.5$ throughout.
In our synthetic experiments, we fix the number of the good arms to be 2 and set their means to be $\mu_1 = \xi + 0.1, \mu_2 = \xi + 0.05$, respectively.
We tested three values of $K$: $4,10$, and $20$, corresponding to $2,8,$ and  $18$ bad arms, respectively.
We equally distribute the means of the bad arms between $\xi-0.05$ and $\xi-0.1$. Table~\ref{tab:synthetic-setup} describes our setup.


\begin{table}[h]
    \centering
    \begin{tabular}{c|cccc}
        & \multicolumn{4}{c}{Number of arms with $\mu = $}\\
         $K$& $\xi + 0.1$ & $\xi + 0.05$ & $\xi - 0.05$ & $\xi - 0.1$ \\
         \midrule
         4 & 1 & 1 & 1 & 1\\ 
         10 & 1 & 1 & 4 & 4\\
         20 & 1 & 1 & 9 & 9\\
    \end{tabular}
    \caption{Number of arms with specified means in the synthetic experimental setup}
    \label{tab:synthetic-setup}
\end{table}


\subsubsection{Sampling policies}
To test the performance of Algorithm~\ref{alg:i-GAI}, we consider two reward-maximizing sampling algorithms, along with three algorithms considered by
\cite{kano2018good}. Furthermore, we benchmark the performance of all algorithms against an oracle that knows the arm means in advance, deterministically pulling the optimal arm at each round.
The arm selection criteria of each algorithm in each iteration can be compactly described as follows:
\begin{align*}
    &\text{Oracle Policy} = \argmax_{a\in \Ical_{t-1}} \mu_a\\
    &\text{MOSS} (\text{Alg. 2}) = \argmax_{a\in \Ical_{t-1}} 
    \hat\mu_{t-1}(a) + \sqrt{\frac{1 + \alpha}{2} \frac{\max(0, \log{\frac{t}{K N_{t-1}(a)})}}{N_{t-1}(a)}}\\
    &\text{UCB} = \argmax_{a\in \Ical_{t-1}}
    \hat\mu_{t-1}(a) + \sqrt{\frac{\log(1 + t \log^2(t)))}{2 N_{t-1}(a)}}\\
    &\text{HDoC} = \argmax_{a\in \Ical_{t-1}}
    \hat\mu_{t-1}(a) + \sqrt{\frac{\log(t)}{2 N_{t-1}(a)}}\\
    &\text{LUCB-G} = \argmax_{a\in \Ical_{t-1}}
     \hat\mu_{t-1}(a) + \sqrt{\frac{\log(4 K N^2_{t-1}(a) / \alpha)}{2 N_{t-1}(a)}}\\
    &\text{APT-G} = \argmin_{a\in \Ical_{t-1}}
    \sqrt{N_{t-1}(a)} |\xi - \hat\mu_{t-1}(a)|\\
\end{align*}
We initialize $\hat\mu_0=\xi=0.5$. 
We set $\alpha = 0.05$ in Alg. 2 (MOSS) and LUCB-G.


\subsubsection{Stopping Criterion}
Alg. 2 in our ablation study implements the stopping criterion as described in
Algorithm \ref{alg:i-GAI}.
We set the error parameter $\delta = 0.05$ and truncation constant $b = 0.98$. We did not specify $m$: we ran the sampling policy until all arms were identified as either good or bad.

The stopping rule in HDoC, LUCB-G, and APT-G are provided by \cite{kano2018good}. We label 
an arm $a$ as good, if \[
\hat\mu_{t-1}(a) + \sqrt{\frac{\log(4KN^2_{t-1}(a) / \delta)}{2 N_{t-1}(a)}} > \xi \] and as bad, if \[
\hat\mu_{t-1}(a) - \sqrt{\frac{\log(4KN^2_{t-1}(a) / \delta)}{2 N_{t-1}(a)}} < \xi. \]
We set $\delta = 0.05$.

Since the oracle knows the true mean vector $\bm\mu$, we substitute the true mean into the stopping criteria for the oracle, rather than using the predictable plug-in 
$e$-process described in ~\autoref{defn:predicable_maximin}.  This substitution is applied in \autoref{eq:lambda_minus} and \autoref{eq:lambda_plus}.



\subsection{Implementation and Runtime}
The runtime of all algorithms is linear in the number of iterations necessary to stop all arms. 
We implement our experiments in Python.
Without parallelization, the runtime of our algorithm is roughly 1/6000 seconds per iteration. A parallelized version of the algorithm was executed 
on an AWS EC2 c7a.12xlarge instance with 48 cores of CPU and 96GiB RAM.
All methods are computationally efficient, as the sampling policies and stopping criterion can all be
incrementally updated with low computational complexity.




\subsection{Ablation studies}
\label{subsec:ablation}
Our algorithm consists of two components: 1)
an adaptive sampling scheme and 2) a sequential AV test that serves as the stopping criterion.
To investigate the performance of each component, we
conduct two ablation studies, testing different combinations of sampling schemes and stopping criteria under two DGPs described in ~\autoref{subsec:dgp}.

\begin{table}[tbh]
    \centering
    \fontsize{8}{7}\selectfont
     \begin{tabular}{llll|rlrlrlrl}
 $K$ & Dist. & Policy & Stopping Criteria & \multicolumn{2}{c}{$\tau_{\Gcal, 1}$} & \multicolumn{2}{c}{$\tau_{\Gcal, 2}$} & \multicolumn{2}{c}{$\tau_{\text{stop}}$} & \multicolumn{2}{c}{$R[\tau_{\Gcal, 1}]$} \\
 \toprule
\multirow[c]{8}{*}{4} & \multirow[c]{4}{*}{Bern} & Alg. 2 & Alg. 2 & 532.8 $\pm$ & 361.6 & 1954.8 $\pm$ & 892.1 & 3588.6 $\pm$ & 1213.4 & 19.0 $\pm$ & 17.1 \\
 &  & HDoC & Alg. 2 & 570.4 $\pm$ & 324.3 & 2086.4 $\pm$ & 947.0 & 3605.6 $\pm$ & 1244.2 & 22.4 $\pm$ & 13.0 \\
 &  & Alg. 2 & HDoC & 1231.2 $\pm$ & 536.4 & 5623.2 $\pm$ & 1587.9 & 10707.5 $\pm$ & 2180.5 & 24.9 $\pm$ & 19.8 \\
 &  & HDoC & HDoC & 1422.3 $\pm$ & 552.5 & 5692.7 $\pm$ & 1556.7 & 10729.0 $\pm$ & 2125.1 & 40.0 $\pm$ & 18.8 \\
\cmidrule(lr){2-12}
 & \multirow[c]{4}{*}{Mix} & Alg. 2 & Alg. 2 & 355.1 $\pm$ & 201.1 & 1279.0 $\pm$ & 580.4 & 2366.4 $\pm$ & 837.0 & 13.6 $\pm$ & 9.9 \\
 &  & HDoC & Alg. 2 & 399.3 $\pm$ & 231.9 & 1370.2 $\pm$ & 554.6 & 2394.0 $\pm$ & 772.1 & 18.0 $\pm$ & 10.7 \\
 &  & Alg. 2 & HDoC & 1199.4 $\pm$ & 391.9 & 5530.2 $\pm$ & 1302.4 & 10949.5 $\pm$ & 1833.1 & 22.9 $\pm$ & 12.4 \\
 &  & HDoC & HDoC & 1407.5 $\pm$ & 458.0 & 5629.3 $\pm$ & 1286.4 & 10856.6 $\pm$ & 1845.3 & 38.9 $\pm$ & 15.5 \\
\cmidrule(lr){1-12} \cmidrule(lr){2-12}
\multirow[c]{8}{*}{10} & \multirow[c]{4}{*}{Bern} & Alg. 2 & Alg. 2 & 827.4 $\pm$ & 579.3 & 2596.9 $\pm$ & 1006.4 & 10319.4 $\pm$ & 2055.9 & 50.5 $\pm$ & 39.5 \\
 &  & HDoC & Alg. 2 & 904.3 $\pm$ & 506.2 & 3173.3 $\pm$ & 1339.2 & 10362.8 $\pm$ & 1994.9 & 72.5 $\pm$ & 37.4 \\
 &  & Alg. 2 & HDoC & 1562.0 $\pm$ & 829.3 & 6329.0 $\pm$ & 1800.0 & 27977.4 $\pm$ & 3431.7 & 60.5 $\pm$ & 50.2 \\
 &  & HDoC & HDoC & 1898.3 $\pm$ & 635.2 & 7048.2 $\pm$ & 1664.1 & 27999.0 $\pm$ & 3464.8 & 112.1 $\pm$ & 36.7 \\
\cmidrule(lr){2-12}
 & \multirow[c]{4}{*}{Mix} & Alg. 2 & Alg. 2 & 513.6 $\pm$ & 293.6 & 1756.9 $\pm$ & 680.8 & 6921.3 $\pm$ & 1239.2 & 33.1 $\pm$ & 20.8 \\
 &  & HDoC & Alg. 2 & 686.9 $\pm$ & 320.9 & 2222.3 $\pm$ & 929.4 & 6877.5 $\pm$ & 1458.8 & 60.3 $\pm$ & 26.2 \\
 &  & Alg. 2 & HDoC & 1512.4 $\pm$ & 693.1 & 6157.6 $\pm$ & 1313.4 & 28457.3 $\pm$ & 3054.7 & 54.8 $\pm$ & 38.3 \\
 &  & HDoC & HDoC & 1894.0 $\pm$ & 599.8 & 7033.4 $\pm$ & 1454.9 & 28452.0 $\pm$ & 2805.4 & 110.7 $\pm$ & 33.2 \\
\cmidrule(lr){1-12} \cmidrule(lr){2-12}
\multirow[c]{8}{*}{20} & \multirow[c]{4}{*}{Bern} & Alg. 2 & Alg. 2 & 1085.8 $\pm$ & 728.1 & 3480.6 $\pm$ & 1758.1 & 22417.2 $\pm$ & 2864.7 & 85.4 $\pm$ & 61.2 \\
 &  & HDoC & Alg. 2 & 1481.9 $\pm$ & 719.8 & 4825.2 $\pm$ & 1652.2 & 22549.2 $\pm$ & 2864.7 & 162.7 $\pm$ & 72.8 \\
 &  & Alg. 2 & HDoC & 2131.0 $\pm$ & 1426.7 & 7055.6 $\pm$ & 1964.7 & 57938.8 $\pm$ & 4915.1 & 120.2 $\pm$ & 100.3 \\
 &  & HDoC & HDoC & 2606.2 $\pm$ & 862.2 & 9263.8 $\pm$ & 2366.6 & 57720.3 $\pm$ & 4800.7 & 232.6 $\pm$ & 70.2 \\
\cmidrule(lr){2-12}
 & \multirow[c]{4}{*}{Mix} & Alg. 2 & Alg. 2 & 678.1 $\pm$ & 409.7 & 2219.0 $\pm$ & 819.1 & 15123.7 $\pm$ & 2066.8 & 58.1 $\pm$ & 36.2 \\
 &  & HDoC & Alg. 2 & 1289.6 $\pm$ & 619.3 & 3864.7 $\pm$ & 1342.0 & 15398.2 $\pm$ & 2102.4 & 150.2 $\pm$ & 67.0 \\
 &  & Alg. 2 & HDoC & 2071.0 $\pm$ & 1283.7 & 7076.4 $\pm$ & 1410.0 & 58546.1 $\pm$ & 3831.8 & 111.8 $\pm$ & 85.0 \\
 &  & HDoC & HDoC & 2674.0 $\pm$ & 687.7 & 9413.8 $\pm$ & 1687.0 & 58446.1 $\pm$ & 3700.6 & 238.8 $\pm$ & 56.5 \\
\cmidrule(lr){1-12} \cmidrule(lr){2-12}
\end{tabular}
\caption{Ablation study for comparing Alg.2 with HDoC by testing different combinations of sampling policies and stopping criterion. We report the mean and standard deviation of 200 independent runs. Across all runs, we observe $2$ runs with mislabeled arms. They were discarded in the stopping time calculation.}
\label{tab:stopping-ablation}
\end{table}

\begin{figure}[ht]
    \centering
    \begin{subfigure}[t]{0.48\linewidth}
        \centering
        \includegraphics[width=\linewidth]{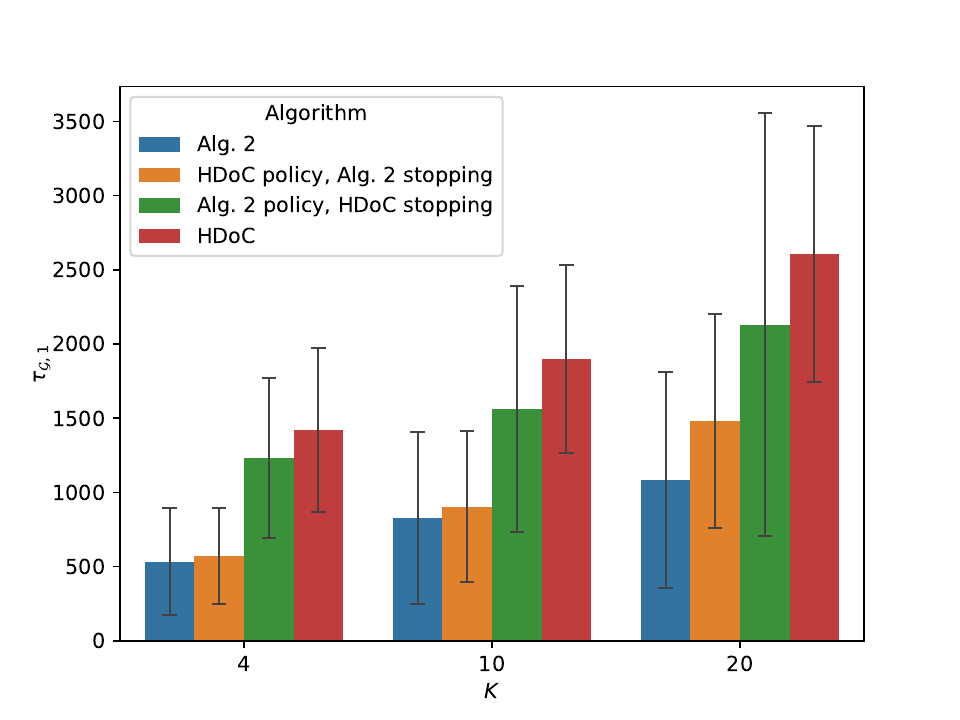}
        \caption{Bernoulli Case}
        \label{fig:bern-stopping-ablation}
    \end{subfigure}
    \hfill
    \begin{subfigure}[t]{0.48\linewidth}
        \centering
        \includegraphics[width=\linewidth]{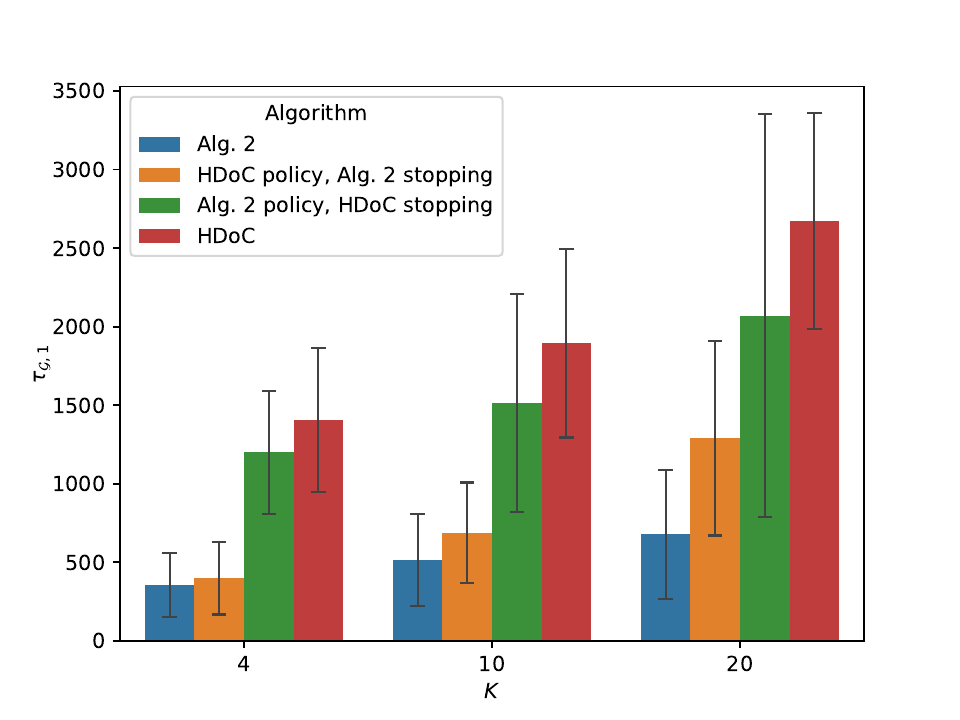}
        \caption{Mixture Case}
        \label{fig:mix-stopping-ablation}
    \end{subfigure}
    \caption{Visualization of $\tau_{\Gcal, 1}$ ablation results from \autoref{tab:stopping-ablation}.}
    \label{fig:stopping-ablation}
\end{figure}

\subsubsection{Comparing HDoC and Alg. 2}
To investigate the relative performance of HDoC and Alg. 2 regarding both the efficiency of the sampling policy and the power of stopping criteria, we exhaustively test the combinations of the 2 sampling policies and stopping criteria in \autoref{tab:stopping-ablation}. Additionally, we provide a visualization of 
%
$\tau_{\Gcal, 1}$ in \autoref{fig:stopping-ablation}.

We observe that when the stopping criterion is fixed to be that of Alg. 2, the performance difference between the two sampling policies is relatively small across all stopping times. However, this gap increases with the number of arms. For $\tau_{\Gcal,1}$, this gap ranges from 40 when $K=4$ and 400 when $K=20$ in the Bernoulli case.
When the stopping criterion is fixed to that of HDoC, we empirically observe that Moss (Alg. 2) outperforms the sampling algorithm of HDoC by a larger margin. 

Further, we observe that our sampling algorithm results in lower regret (measured by $R_{\tau_{\Gcal,1}}$), with the reduction in regret increasing as the number of arms increases. When comparing the 2 mix-and-match experiments, we observe that while HDoC sampling combined with Alg. 2 stopping achieves a lower stopping time, Alg. 2 sampling paired with HDoC stopping achieves a lower regret when $K\geq 10$, highlighting the benefits of regret minimizing algorithms.

When we fix the sampling algorithm and compare the performance of the two stopping criteria, we find that our proposed stopping criterion consistently achieves better stopping times by a significant margin.


Overall, we find that Alg. 2 outperforms all other algorithm configurations. While our improved stopping criterion primarily contributes to the performance gain in stopping time, it's important to note that the regret-minimizing approach achieves lower regret while maintaining comparable stopping time performance. Further, we observe that our modified regret-minimizing policy performs better even when paired with the stopping criterion of HDoC.


\begin{table}[bt]
    \centering
    \fontsize{8}{7}\selectfont
    \begin{tabular}{llll|rlrlrlrl}
 $K$ & Dist. & Policy & Stopping Criteria & \multicolumn{2}{c}{$\tau_{\Gcal, 1}$} & \multicolumn{2}{c}{$\tau_{\Gcal, 2}$} & \multicolumn{2}{c}{$\tau_{\text{stop}}$} & \multicolumn{2}{c}{$R[\tau_{\Gcal, 1}]$}\\
 \toprule
\multirow[c]{8}{*}{4} & \multirow[c]{4}{*}{Bern} & Oracle & Oracle & 255.1 $\pm$ & 177.5 & 1231.3 $\pm$ & 595.5 & 2505.6 $\pm$ & 896.5 & 0.0 $\pm$ & 0.0 \\
 &  & Oracle & Alg. 2 & 353.4 $\pm$ & 222.0 & 1776.1 $\pm$ & 800.4 & 3592.4 $\pm$ & 1210.6 & 0.0 $\pm$ & 0.0 \\
 &  & Alg. 2 & Oracle & 427.2 $\pm$ & 292.1 & 1404.1 $\pm$ & 758.8 & 2520.9 $\pm$ & 922.2 & 16.7 $\pm$ & 15.4 \\
 &  & Alg. 2 & Alg. 2 & 532.8 $\pm$ & 361.6 & 1954.8 $\pm$ & 892.1 & 3588.6 $\pm$ & 1213.4 & 19.0 $\pm$ & 17.1 \\
\cmidrule(lr){2-12}
 & \multirow[c]{4}{*}{Mix} & Oracle & Oracle & 190.7 $\pm$ & 85.1 & 939.5 $\pm$ & 341.0 & 1963.4 $\pm$ & 528.0 & 0.0 $\pm$ & 0.0 \\
 &  & Oracle & Alg. 2 & 214.5 $\pm$ & 120.7 & 1133.3 $\pm$ & 531.0 & 2386.6 $\pm$ & 802.7 & 0.0 $\pm$ & 0.0 \\
 &  & Alg. 2 & Oracle & 311.1 $\pm$ & 175.0 & 1051.9 $\pm$ & 387.2 & 1911.4 $\pm$ & 513.7 & 12.7 $\pm$ & 9.7 \\
 &  & Alg. 2 & Alg. 2 & 355.1 $\pm$ & 201.1 & 1279.0 $\pm$ & 580.4 & 2366.4 $\pm$ & 837.0 & 13.6 $\pm$ & 9.9 \\
\cmidrule(lr){1-12} \cmidrule(lr){2-12}
\multirow[c]{8}{*}{10} & \multirow[c]{4}{*}{Bern} & Oracle & Oracle & 290.6 $\pm$ & 169.8 & 1500.3 $\pm$ & 732.6 & 7578.7 $\pm$ & 1685.0 & 0.0 $\pm$ & 0.0 \\
 &  & Oracle & Alg. 2 & 398.7 $\pm$ & 217.5 & 2051.0 $\pm$ & 884.5 & 10236.3 $\pm$ & 1969.1 & 0.0 $\pm$ & 0.0 \\
 &  & Alg. 2 & Oracle & 616.4 $\pm$ & 394.7 & 1952.4 $\pm$ & 798.6 & 7535.7 $\pm$ & 1521.0 & 41.6 $\pm$ & 28.4 \\
 &  & Alg. 2 & Alg. 2 & 827.4 $\pm$ & 579.3 & 2596.9 $\pm$ & 1006.4 & 10319.4 $\pm$ & 2055.9 & 50.5 $\pm$ & 39.5 \\
\cmidrule(lr){2-12}
 & \multirow[c]{4}{*}{Mix} & Oracle & Oracle & 210.3 $\pm$ & 80.7 & 1099.3 $\pm$ & 354.6 & 5704.8 $\pm$ & 860.6 & 0.0 $\pm$ & 0.0 \\
 &  & Oracle & Alg. 2 & 261.8 $\pm$ & 128.3 & 1310.6 $\pm$ & 545.0 & 6900.7 $\pm$ & 1291.9 & 0.0 $\pm$ & 0.0 \\
 &  & Alg. 2 & Oracle & 489.7 $\pm$ & 263.3 & 1522.1 $\pm$ & 498.6 & 5701.6 $\pm$ & 899.6 & 34.3 $\pm$ & 22.3 \\
 &  & Alg. 2 & Alg. 2 & 513.6 $\pm$ & 293.6 & 1756.9 $\pm$ & 680.8 & 6921.3 $\pm$ & 1239.2 & 33.1 $\pm$ & 20.8 \\
\cmidrule(lr){1-12} \cmidrule(lr){2-12}
\multirow[c]{8}{*}{20} & \multirow[c]{4}{*}{Bern} & Oracle & Oracle & 317.2 $\pm$ & 154.4 & 1679.2 $\pm$ & 700.9 & 16870.0 $\pm$ & 2480.8 & 0.0 $\pm$ & 0.0 \\
 &  & Oracle & Alg. 2 & 433.3 $\pm$ & 204.4 & 2327.9 $\pm$ & 1016.6 & 22484.2 $\pm$ & 2803.2 & 0.0 $\pm$ & 0.0 \\
 &  & Alg. 2 & Oracle & 939.9 $\pm$ & 622.4 & 2833.0 $\pm$ & 1526.8 & 16819.9 $\pm$ & 2490.7 & 82.6 $\pm$ & 57.4 \\
 &  & Alg. 2 & Alg. 2 & 1085.8 $\pm$ & 728.1 & 3480.6 $\pm$ & 1758.1 & 22417.2 $\pm$ & 2864.7 & 85.4 $\pm$ & 61.2 \\
\cmidrule(lr){2-12}
 & \multirow[c]{4}{*}{Mix} & Oracle & Oracle & 245.7 $\pm$ & 85.2 & 1238.2 $\pm$ & 406.4 & 12754.6 $\pm$ & 1274.8 & 0.0 $\pm$ & 0.0 \\
 &  & Oracle & Alg. 2 & 289.6 $\pm$ & 143.3 & 1451.2 $\pm$ & 545.5 & 15216.0 $\pm$ & 2164.5 & 0.0 $\pm$ & 0.0 \\
 &  & Alg. 2 & Oracle & 683.2 $\pm$ & 413.8 & 1956.5 $\pm$ & 603.0 & 12712.6 $\pm$ & 1299.0 & 60.3 $\pm$ & 37.8 \\
 &  & Alg. 2 & Alg. 2 & 678.1 $\pm$ & 409.7 & 2219.0 $\pm$ & 819.1 & 15123.7 $\pm$ & 2066.8 & 58.1 $\pm$ & 36.2 \\
\cmidrule(lr){1-12} \cmidrule(lr){2-12}
\end{tabular}
    \caption{Results for different combinations of policies and stopping times with and without oracle knowledge. We report the mean and standard deviation of 200 independent runs. All arms were correctly identified.}
    \label{tab:oracle-ablation}
\end{table}
\begin{figure}[ht]
    \centering
    \begin{subfigure}[t]{0.48\linewidth}
        \centering
        \includegraphics[width=\linewidth]{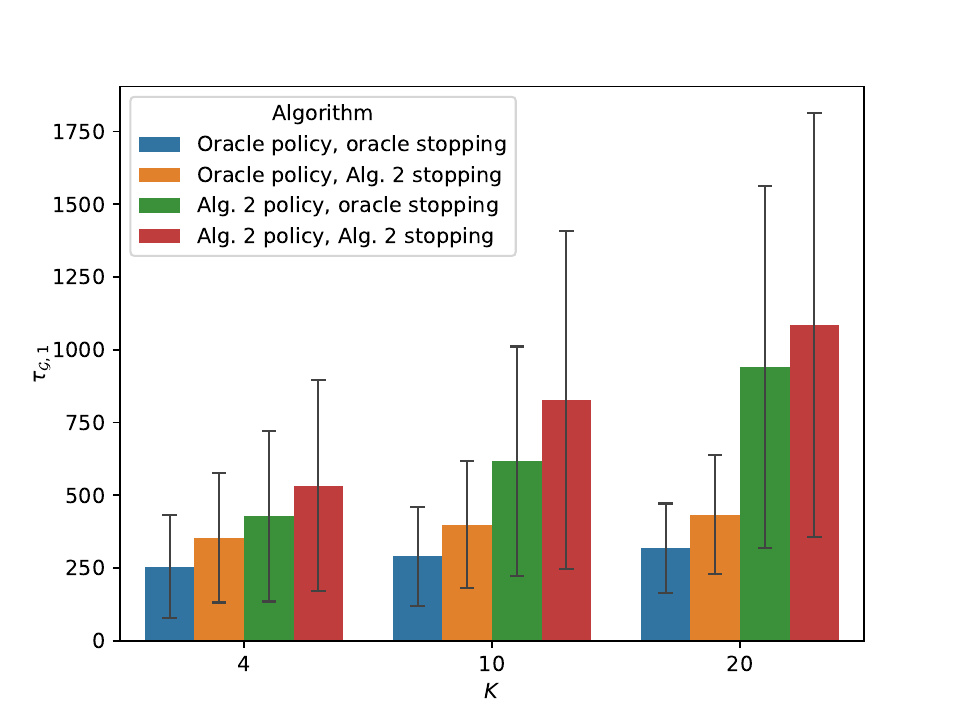}
        \caption{Bernoulli Case}
        \label{fig:bern-oracle-ablation}
    \end{subfigure}
    \hfill
    \begin{subfigure}[t]{0.48\linewidth}
        \centering
        \includegraphics[width=\linewidth]{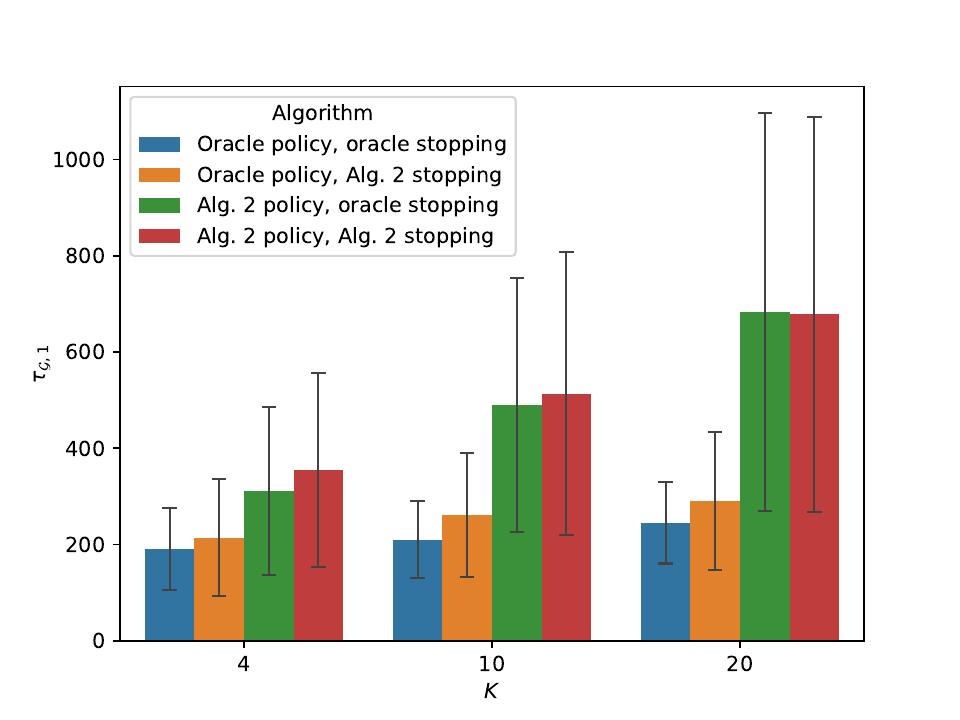}
        \caption{Mixture Case}
        \label{fig:mix-oracle-ablation}
    \end{subfigure}
    \caption{Visualization of $\tau_{\Gcal, 1}$ results from \autoref{tab:oracle-ablation}.}
    \label{fig:oracle-ablation}
\end{figure}

\subsubsection{Comparing Oracle with Alg. 2}
Next, we compare the performance of Alg. 2 to that of Oracle to gain insights into how closely our sampling policy and stopping criterion align with the performance achievable with oracle knowledge.
%
%
%
We present our results in \autoref{tab:oracle-ablation}, and
visualized $\tau_{\Gcal, 1}$
in \autoref{fig:oracle-ablation}.

When fixing the sampling policy to be either the oracle or Alg. 2, we observe that our stopping criterion, on average, incurs an additional 100 draws for $\tau_{\Gcal,1}$ and 500 draws for $\tau_{\Gcal,1}$ under Bernoulli distributions. Under mixture distributions, this gap is further reduced to 50 and 220, respectively.


When fixing the stopping criterion to be either the oracle or Alg. 2, we observe that our sampling policy incurs an additional 200 draws when $K=4$, with this number increasing as the number of arms grows.
However, this gap is significantly smaller than that between HDoC and the Oracle, as evidenced by the comparisons between \autoref{tab:oracle-ablation} and \autoref{tab:stopping-ablation}. 



\newpage

\section{Proofs}
Throughout our proofs, we will build upon two simple yet useful lemmas for establishing our theoretical results.

\subsection{Preliminary Lemmas}

\begin{lemma}[Concavity with Respect to $\lambda$.]\label{lem:concavity}

The function $\log(1+\lambda(X - \xi))$ is concave with respect to $X \in [0,1]$ for any $\xi \in (0,1)$, and any $\lambda \in [\frac{-b}{1-\xi}, \frac{b}{\xi}]$ for $b\in (0,1)$. 
\end{lemma}

\begin{proof}[Proof of Lemma \ref{lem:concavity}]
    We show this with a simple second derivative test. 
    \begin{align}
        \frac{\partial^2}{\partial^2 X}\left( \log(1+\lambda(X-\xi) \right) &= \frac{\partial}{\partial \lambda} \frac{\lambda}{1+\lambda(X-\xi)}\\
        &= -\frac{\lambda^2}{(1+\lambda(X-\xi))^2}
    \end{align}
    By the bounds on $\lambda$, $b$, $X$, and $\xi$, the denominator term is nonzero, and the second derivative is well defined. 
    $$ \lambda(X - \xi) \geq \min\left(\frac{-b}{1-\xi}(1-\xi), \frac{b}{\xi}(-\xi)\right) = -b > -1.  $$
    Thus, we obtain that the second derivative is strictly nonpositive, i.e., $\frac{\partial^2}{\partial^2 \lambda}\left( \log(1+\lambda(X-\xi) \right) \leq 0$, and thus $\log(1+\lambda(X - \xi))$ is concave with respect to $\lambda$ for any $X, \xi \in [0,1]$, and $b \in (0,1)$. 
\end{proof}

\begin{lemma}[Worst-Case Instance for $\Pcal(\bm \mu)$]\label{lem:worst_case_instance}

For any $m_a \in [0,1]$, any $\xi \in (0,1)$, and any $\lambda \in [\frac{-b}{1-\xi}, \frac{b}{\xi}]$ for $b\in (0,1)$
$$\inf_{P_a \in \Pcal(\mu_a)}\EE_{X \sim P_a}[\log(1+\lambda(X-\xi))] = \EE_{X \sim \text{Bern}(\mu_a)}[\log(1+\lambda(X - \xi))].$$
\end{lemma}

\begin{proof}[Proof of Lemma \ref{lem:worst_case_instance}]
    We leverage the results of Lemma \ref{lem:concavity} to prove Lemma \ref{lem:worst_case_instance}. For any $P_a \in \Pcal(m_a)$, we construct the Bernoulli random variable $R = [U \leq X]$, where $U \sim \text{Unif}[0,1]$ is an independent uniformly distributed random variable. Note that $\EE[R|X] = \EE[\text{1}[U \leq X]|X] = X$, and thus $\EE[R]= \EE_X\EE[R|X] = \mu_a$ for $X \sim P_a$, for all $P_a \in \Pcal(m_a)$. Then, by the concavity of the function $f(x)= \log(1+\lambda(x-\xi))$ w.r.t. $x$ and Jensen's inequality,
    $$ \EE[f(R)|X] \leq f(\EE[f(R)|X]) \leq f(X) \implies \EE[f(R)] \leq \EE_{X \sim P_a}[f(X)], \quad \forall P_a \in \Pcal(m_a).  $$
    We conclude the proof  that $R$ is just a Bernoulli random variable with mean $\mu_a$, which concludes the proof. 
\end{proof}

\subsection{Proof of Theorem \ref{thm:minmax_opt_e_power}}\label{app:proof_thm_minmax_opt_e_power}

We aim to provide a lower bound for the following term for a fixed number of arm pulls $N_t(a)$ of arm $a$:
$$\inf_{P \in \Pcal(\mu_a)} \sup_{E = (E_t)_{t \in \NN}} \frac{\EE[\log(E_t)]}{N_t(a)}$$
We focus on the case where $\mu_a < \xi$, and the null hypothesis class $\Hcal_a^+ = \{P \in \Pcal(\mu_a): \mu_a \leq \xi\}$. The proof for $\mu_a > \xi$ is symmetric, and can be reproduced directly by following the same steps.

\paragraph{Proof Sketch.} First, we reduce the inner supremum using existing results relating $e$-processes and test martingales. We do so by providing (i) the universal representation of test martingales under our nonparametric assumptions and (ii) the admissibility of $e$-processes relative to test martingales. Second, using Lemma \ref{lem:worst_case_instance}, we show that the $e$-power is bounded by the lower bound presented in Theorem \ref{thm:minmax_opt_e_power}.

\begin{lemma}[Universal Representation of Test Martingales, Proposition 2 of \citealt{waudbysmith2022estimating}]\label{lem:universal_rep}
    We say that $M = (M_t)_{t \in \NN}$ is a nonnegative test supermartingale for null hypothesis $\Hcal$ if it satisfies the following properties: (i) $M_t \geq 0$ for all $t \in \NN$, (ii) $\EE[M_t|M_{t-1}] \leq M_{t-1} $, and (iii) $M_0 = 1$. 
    
    For any test supermartingale $M$ for $\Hcal^+_a = \{P \in \Pcal(\mu_a): \mu_a > \xi\}$ (or $\Hcal^-_a= \{P \in \Pcal(\mu_a): \mu_a \leq \xi\}$), $M = (M_t)_{t \in \NN}$ is only a test martingale if and only if $M_t = \prod_{i=1}^t \left(1+\lambda_t (X_t - \xi)  \right)$ for some $\Fcal_{t-1}$-measurable $\lambda_t \in [0, -\frac{1}{1-\xi}]$ (equivalently, $\lambda_t \in [0,\frac{1}{\xi}]$) for all $t \in \NN$. 
\end{lemma}

\begin{lemma}[Admissibility of $e$-Processes, Lemma 6 of \citealt{ramdas2022admissible}]\label{lem:domination}
    Let $E$ be an $e$-process as defined in Definition \ref{defn:e_process} with the null $\Hcal$. Then, there exists a nonnegative test supermartingale $M$ with respect to the null $\Hcal$ such that $E$ is upper bounded by $M$ with probability 1, i.e., $\forall t \in \NN$, $E_t \leq M_t$ almost surely. 
\end{lemma}

Combining these results, this says that for any $e$-process with respect to $\Hcal_a^+$, a test martingale in the class $\Ecal(\Hcal_a^+, 1)$ upper bounds its value. Thus, by the monotonicity of $\log(\cdot)$, the supremum of any possible $e$-process $E$ for the null $\Hcal_a^+$ is equivalent to the supremum attained by a test martingale in the class $\Ecal(\Hcal_a^+, 1)$: 
\begin{align}\label{eq:sup_e_equal_sup_m}
    \sup_{E = (E_t)_{t \in \NN} }\frac{\EE[\log(E_t)]}{N_t(a)} &= \sup_{E_a = (E_t^a(\lambda, \xi))_{t\in\NN} \in \Ecal(\Hcal_a^+, 1)} \frac{\EE[\log(E_t)]}{N_t(a)} \\
    &= \sup_{(\lambda_t)_{t \in \NN}\in [-\frac{1}{1-\xi}, 0]^\infty : \lambda_t \text{ is $\Fcal_{t-1}$ measurable}} \frac{\EE\left[\sum_{i: A_i = a}\log\left(1+\lambda_i(X_i-\xi) \right)\right]}{N_t(a)}
\end{align}
Thus, our initial sum can be re-expressed as follows, where $\gamma_r = \inf\{i \in [t]: N_i(a) = r\}$ is the (random) first time in which we have pulled the $a$-th arm $r$ number of times.
\begin{align}
    \inf_{P \in \Pcal(\mu_a)} \sup_{E = (E_t)_{t \in \NN}} \frac{\EE[\log(E_t)]}{N_t(a)} 
    &= \inf_{P \in \Pcal(\mu_a)}  \sup_{(\lambda_t)_{t \in \NN}\in [-\frac{1}{1-\xi}, 0]^\infty : \lambda_t \text{ is $\Fcal_{t-1}$ measurable}} \frac{\EE\left[\sum_{i: A_i = a}\log\left(1+\lambda_i(X_i-\xi) \right)\right]}{N_t(a)}\\
    &= \inf_{P \in \Pcal(\mu_a)} \sup_{(\lambda_t)_{t \in \NN}\in [-\frac{1}{1-\xi}, 0]^\infty : \lambda_t \text{ is $\Fcal_{t-1}$ measurable}} \frac{\sum_{r = 1}^{N_t(a)}\EE[\log\left(1+\lambda_{\gamma_r }(X_{\gamma_r } - \xi) \right)]}{N_t(a)}\\
    & \leq \frac{1}{N_t(a)}\inf_{P \in \Pcal(\mu_a)} \sum_{r=1}^{N_t(a)}\sup_{\lambda_t\in [-\frac{1}{1-\xi}, 0] : \lambda_t \text{ is $\Fcal_{t-1}$ measurable}} \EE[\log(1+\lambda_{\gamma_r }(X_{\gamma_r } - \xi))] \label{line:above}\\
    &\leq \frac{1}{N_t(a)}\inf_{P \in \Pcal(\mu_a)} \sum_{r=1 }^{N_t(a)}\max_{\lambda_t\in [-\frac{1}{1-\xi}, 0] : \lambda_t \text{ is $\Fcal_{t-1}$ measurable}} \EE[\log(1+\lambda_{\gamma_r }(X_{\gamma_r } - \xi))]
\end{align}

By indexing by $\gamma_r$, we move the expectation into the summation term without the need for indicator functions. We obtain the inequality in the line \ref{line:above} due to $\sup \sum_i X_i \leq \sum_i \sup X_i$. For all $P \in \Pcal(\mu_a)$, $\lambda_{\gamma_r } \in [-\frac{1}{1-\xi}, 0]$, the inner supremum is indeed a maximum achieved by some $\lambda^* = \argmax_{\lambda \in [-\frac{1}{1-\xi}, 0]} \EE_P[\log(1+(\lambda_{\gamma_r})(X_{\gamma_r} - \xi)]$ because (i) $\EE[\log(1+\lambda_{\gamma_r}(X_{\gamma_R} - \xi))]$ is a continuous function of $\lambda$ bounded from above, and (ii) $\lambda_t$ is contained within a closed, bounded range. For any distributional sequence $P \in \Pcal(\mu_a)$, let $\lambda_{\gamma_r}^*(P) = \argmax_{\lambda \in [-\frac{1}{1-\xi}, 0]} \EE_P[\log(1+(\lambda_{\gamma_r})(X_{\gamma_r} - \xi)]$, where the dependence is specified by $P$. Now, note that by Lemma \ref{lem:worst_case_instance},
$$ 
\forall \ \gamma_r \in [t], \ \forall \lambda_{\gamma_r} \in \left[0, -\frac{1}{1-\xi}\right], \ \inf _{P \in \Pcal(\mu_a)}\EE[\log(1+\lambda_{\gamma_r}(X_{\gamma_r} - \xi))] \geq \EE_{X \sim \text{Bern}(\mu_a)}[\log(1+ \lambda_{\gamma_r}(X-\xi)].
$$ 
Because $\lambda_{\gamma_r}^*(P) \in [0, -\frac{1}{1-\xi}]$ for all $\gamma_r \in [t]$, $P \in \Pcal$, we immediately obtain that

\begin{align}
    \inf_{P \in \Pcal(\mu_a)} \sup_{E = (E_t)_{t \in \NN}} \frac{\EE[\log(E_t)]}{N_t(a)} &\leq \frac{1}{N_t(a)}\inf_{P \in \Pcal(\mu_a)} \sum_{r=1 }^{N_t(a)}\max_{\lambda_t\in [-\frac{1}{1-\xi}, 0] : \lambda_t \text{ is $\Fcal_{t-1}$ measurable}} \EE[\log(1+\lambda_{\gamma_r }(X_{\gamma_r } - \xi))]\\
    &\leq \frac{1}{N_t(a)} \sum_{r=1}^{N_t(a)} \max_{\lambda_t \in [-\frac{1}{1-\xi}, 0]} \EE_{X \sim \text{Bern}(\mu_a)}[\log(1+ \lambda_{\gamma_r}(X - \xi))]  \\ 
    &= \frac{1}{N_t(a)}\sum_{r=1}^{N_t(a)} \log\frac{1-\mu_a}{1-\xi} + \mu_a \log\frac{\mu_a(1-\xi)}{\xi(1-\mu_a)}\\
    &=  \log\frac{1-\mu_a}{1-\xi} + \mu_a \log\frac{\mu_a(1-\xi)}{\xi(1-\mu_a)}, 
\end{align}
where the maximum is achieved at $\lambda_{\gamma_r} = \lambda_{\text{opt}} = \frac{\mu_a - \xi}{\xi(1-\xi)}$ for all $t \in \NN$. Because $P = (P_t = P_{\text{Bern}(\mu_a)})_{t \in \NN} \in \Pcal(\mu_a)$, and $E = \left(E_t = \prod_{t: A_t = a} (1+\lambda_{\text{opt}}(X_i - \mu)) \right)_{t \in \NN}$ is a valid $e$-process for $\Hcal_{a}^+$, our inequality can be changed to an equality, giving the  desired result.

\subsection{Proof of Theorem \ref{thm:minmax_opt_gen_bern}}\label{proof:thm_2_e_power_optimal}

Theorem \ref{thm:minmax_opt_gen_bern} provides a lower bound on the following quantity:
$$\inf_{P \in \Pcal(\mu_a)}\frac{\EE[\log\left( E_t^{PrPl}(\Hcal_a^-, b) \right)]}{N_t(a)}.$$
We focus on the case where $\mu_a > \xi$, and the null hypothesis class $\Hcal_a^- = \{P \in \Pcal(\mu_a): \mu_a \leq \xi\}$. The proof for $\mu_a < \xi$ is symmetric, and can be reproduced directly by following the same steps. To provide a lower bound, we require validity for any fixed $N_t(a)$ under a potentially data-adaptive distribution. To obtain guarantees under any arbitrary sampling schemes for a fixed number of pulls $N_t(a)$, we use the following anytime valid confidence sequence from \cite{Howard_2021}.

\begin{lemma}[AV Confidence Interval for $1$-Subgaussian Observations, \citealt{Howard_2021}]\label{lem:av_inequality_howard}
    Let $N_t(a) = \sum_{i=1}^t \mathbf{1}[A_i  = a]$ be the number of observations from arm $a_i$ up to time $t$, and let $\hat\mu_t(a) = N_t(a)^{-1}\sum_{i:A_i = a} X_i$. Then, under any sampling scheme $\pi = (\pi_t)_{t \in \NN}$ where $\pi_t$ is $\Fcal_{t-1}$-measurable and any $P \in \Pcal(\mu_a)$,  
    $$ P\left(|\hat\mu_t(a) - \mu_a| \geq 1.7 \sqrt{\frac{\log\log(2N_t(a)) + 0.72\log(10.4/\alpha)}{N_t(a)}}\right) \leq \alpha.  $$
\end{lemma}

We note by using a $1$-subgaussian confidence interval, our bound shown in Theorem \ref{thm:minmax_opt_gen_bern} is loose. However, this still provides asymptotic optimality results and suffices for our proof. 

\paragraph{Proof Outline.} Our proof proceeds in three steps. First, we separate the infimum across our observations, and use Lemma \ref{lem:worst_case_instance} to lower-bound each expectation by its worst-case value. We then use the concavity of the function $f(\mu) = \log(1+\frac{\mu - \xi}{\xi (1-\xi)}(X-\xi))$ with respect to $\mu$ to bound the difference between expectations using $\hat\mu_{t-1}(a)$ and $\mu_a$ as a function of $|\hat\mu_{t-1}(a) - \mu_a|$. Lastly, we use the anytime valid confidence bounds shown in Lemma \ref{lem:av_inequality_howard} to provide our desired result. 

\paragraph{Step 1.} We first upper bound our term using a reindexing of our summation, where $\gamma_r = \inf\{i \in [t]: N_i(a) = r\}$ is the (random) first time in which we have pulled the $a$-th arm $r$ number of times.
\begin{align}
    \inf_{P \in \Pcal(\mu_a)}\frac{\EE[\log\left( E_t^{PrPl}(\Hcal_a^-, b) \right)]}{N_t(a)} &= \frac{1}{N_t(a)}\inf_{P \in \Pcal(\mu_a)}  \EE[\sum_{i: A_i = a} \log(1 + \lambda_{i,a}^-(X_i - \xi))]\\
    &= \frac{1}{N_t(a)} \inf_{P \in \Pcal(\mu_a)} \EE[\sum_{r = 1}^{N_t(a)} \log\left(1+ \lambda_{\gamma_r, a}^- (X_{\gamma_r} - \xi) \right)  ] \\
    &= \frac{1}{N_t(a)} \inf_{P \in \Pcal(\mu_a)} \sum_{r=1}^{N_t(a)} \EE[\log(1+\lambda_{\gamma_r, a}^-(X_{\gamma_r} - \xi))] \label{line:moving_sum} \\ 
    &\geq \frac{1}{N_t(a)} \sum_{r=1}^{N_t(a)} \inf_{P \in \Pcal(\mu_a)} \EE[\log(1+\lambda_{\gamma_r, a}^-(X_{\gamma_r} - \xi))] \label{line:sum_inf_less}\\
    &\geq \frac{1}{N_t(a)} \sum_{r=1}^{N_t(a)} \inf_{P \in \Pcal(\mu_a)} \EE_{X_{\gamma_r} \sim \text{Bern}(\mu_a)}[\log(1+\lambda_{\gamma_r, a}^-(X_{\gamma_r} - \xi))] \label{line:inf_bern}
\end{align}
Note that the outer summation is deterministic because $N_t(a)$ is fixed, and therefore we can safely move our expectation into the summation in line (\ref{line:moving_sum}). We obtain line (\ref{line:sum_inf_less}) by the simple fact that $\inf \sum_i X_i \geq \sum_i \inf X_i$. Lastly, we use Lemma \ref{lem:worst_case_instance} to get our lower bound, where now only $\lambda_{\gamma_r, a}^-$ depends on $P \in \Pcal(\mu_a)$, and $X_{\gamma_r}$ are all generated with according to a Bernoulli distribution.

\paragraph{Step 2.} We now directly work with $\lambda_{i,a}^-$ to obtain bounds on the plug-in error rate, i.e. the difference between using the oracle $\mu_a$ for $\lambda_{i,a}^-$ instead of our plug in $\hat\mu_{i-1}(a)$. We first provide a local convexity result using the assumption that $\mu \in (\xi, b(1-\xi) + \xi)$ for $b \in (0,1)$. We first show that the function $f(\mu) = \log(1+\frac{\mu - \xi}{\xi (1-\xi)})$  is concave with respect to $\mu$ with the second derivative test for any $\mu \in (\xi, b(1-\xi) + \xi)$:
\begin{align}
    \frac{\partial^2}{\partial^2 \mu} f(\mu) &= \frac{\partial}{\partial \mu}\frac{ \frac{(X - \xi)}{\xi(1-\xi)}}{1 + \frac{(X - \xi)}{\xi(1-\xi)}(\mu - \xi)} = -\frac{\left(\frac{(X - \xi)}{\xi(1-\xi)}\right)^2}{\left(1 + \frac{(X - \xi)}{\xi(1-\xi)}(\mu - \xi)\right)^2} \leq 0
\end{align}
where the denominator is bounded away from 0 by our assumptions. Then, by definition of concavity,
$$|f(\mu_a) - f(\hat\mu_{i-1}(a))| \leq \left|\frac{ \frac{(X - \xi)}{\xi(1-\xi)}}{1 + \frac{(X - \xi)}{\xi(1-\xi)}(\mu - \xi)}\right| \times |\mu_a  - \hat\mu_{i-1}(a)| \leq \frac{1}{b\xi(1-\xi)} |\mu_a  - \hat\mu_{i-1}(a)|. $$

Note that $\lambda_{i,a}^- = \min\left(\frac{b}{\xi},\max( \frac{(\hat\mu_{i-1}(a) - \xi)}{\xi (1-\xi)}, 0 )\right)$ is equivalent to bounding $\hat\mu_{i-1}(a) \in [\xi, b(1-\xi) + \xi]$, and note that
\begin{align}
    \EE[|\hat\mu_{i-1}(a) - \mu_a|]  &= \PP(\hat\mu_{i-1}(a) < \xi) \times \EE[|\hat\mu_{i-1}(a) - \mu_a| \ | \hat\mu_{i-1}(a) < \xi ] \\ 
    &\quad + \PP(\hat\mu_{i-1}(a) > b(1-\xi)+ \xi) \times \EE[|\hat\mu_{i-1}(a) - \mu_a| \ | \hat\mu_{i-1}(a) > b(1-\xi)+ \xi ] \\ 
    &\quad + \PP(\xi \leq \hat\mu_{i-1}(a) \leq b(1-\xi)+ \xi) \times \EE[|\hat\mu_{i-1}(a) - \mu_a|\  | \xi \leq \hat\mu_{i-1}(a) \leq b(1-\xi)+ \xi ] \\ 
    & \geq \PP(\hat\mu_{i-1}(a) < \xi) \times |\mu_a - \xi| \\
    &\quad + \PP(\hat\mu_{i-1}(a) > b(1-\xi)+ \xi) \times |\mu_a - b(1-\xi) + \xi| \\ 
    &\quad + \PP(\xi \leq \hat\mu_{i-1}(a) \leq b(1-\xi)+ \xi) \times \EE[|\hat\mu_{i-1}(a) - \mu_a|\  | \xi \leq \hat\mu_{i-1}(a) \leq b(1-\xi)+ \xi ] \\
    &= \EE[|\max(\xi,\min(\hat\mu_{i-1}(a), b(1-\xi) + \xi)) - \mu_a|].
\end{align}
Thus we obtain the following inequality for our estimator for all $i \in \NN$:
\begin{align}
    \EE\left[|\log(1+ \frac{\mu_a - \xi}{\xi(1-\xi)}(X_i - \xi)) - \log(1 + \lambda_{i,a}^-(X_i - \xi)) |\right] &= \EE\left[|f(\mu_a) - f\left(\max(\xi,\min(\hat\mu_{i-1}(a), b(1-\xi) + \xi))\right)|\right]\\
    &\leq \frac{1}{b\xi(1-\xi)}\EE[|\max(\xi,\min(\hat\mu_{i-1}(a), b(1-\xi) + \xi)) - \mu_a|]\\
    &\leq \frac{1}{b\xi(1-\xi)} \EE[|\hat\mu_{i-1}(a) - \mu_a|].
\end{align}

\paragraph{Step 3.} Finally, we use the results in Lemma \ref{lem:av_inequality_howard} to obtain our desired bound. We first subtract the term $\left(\log\frac{1-\mu_a}{1-\xi} + \mu_a \log\frac{\mu_a(1-\xi)}{\xi(1-\mu_a)}\right)$, the minimax result of Theorem \ref{thm:minmax_opt_e_power}, to obtain the following expression: 

\begin{align}
    &\inf_{P \in \Pcal(\mu_a)}\frac{\EE[\log\left( E_t^{PrPl}(\Hcal_a^-, b) \right)]}{N_t(a)} - \left(\log\frac{1-\mu_a}{1-\xi} + \mu_a \log\frac{\mu_a(1-\xi)}{\xi(1-\mu_a)}\right)  \\ 
    &\quad \geq \frac{1}{N_t(a)} \sum_{r=1}^{N_t(a)} \inf_{P \in \Pcal(\mu_a)} \EE_{X_{\gamma_r} \sim \text{Bern}(\mu_a)}[\log(1+\lambda_{\gamma_r, a}^-(X_{\gamma_r} - \xi))]  - \left(\log\frac{1-\mu_a}{1-\xi} + \mu_a \log\frac{\mu_a(1-\xi)}{\xi(1-\mu_a)}\right)  \\
    &\quad =  \frac{1}{N_t(a)} \sum_{r=1}^{N_t(a)} \inf_{P \in \Pcal(\mu_a)} \EE_{X_{\gamma_r} \sim \text{Bern}(\mu_a)}\left[\log(1+\lambda_{\gamma_r, a}^-(X_{\gamma_r} - \xi)) - \log(1+\frac{\mu_a - \xi}{\xi(1-\xi)}(X_{\gamma_r} - \xi))\right] \\ 
    &\geq \frac{1}{N_t(a)} \sum_{r=1}^{N_t(a)} \inf_{P \in \Pcal(\mu_a)} -\EE_{X_{\gamma_r} \sim \text{Bern}(\mu_a)}\left[\left|\log(1+\lambda_{\gamma_r, a}^-(X_{\gamma_r} - \xi)) - \log(1+\frac{\mu_a - \xi}{\xi(1-\xi)}(X_{\gamma_r} - \xi))\right|\right]\\
    &\geq \frac{1}{N_t(a)} \frac{1}{b\xi(1-\xi)} \sum_{r=1}^{N_t(a)} -\sup_{P \in \Pcal(\mu_a)} \EE[|\hat\mu_{\gamma_r-1}(a) - \mu_a|] \label{line:summation_line},
\end{align}
where the last line comes from the results of Step 2 and the fact that the expression inside the supremum is bounded above under our assumptions. We now bound the term $\sup_{P \in \Pcal(\mu_a)}\EE[|\hat\mu_{\gamma_r-1}(a) - \mu_a|]$. Note that by definition, at time $\gamma_r - 1$, $N_{\gamma_r -1}(a) = r-1$, and all observations $X_{\gamma_1}, ... , X_{\gamma_{r-1}}$ used to form $\hat\mu_{\gamma_r - 1}(a)$ are $1$-subgaussian for all $P \in \Pcal(\mu_a)$. Therefore, by Lemma \ref{lem:av_inequality_howard}, for $r > 1$, 
$$\forall r \in \{1, ..., N_t(a)\}, \ \sup_{P \in \Pcal(\mu_a)} P\left(|\hat\mu_{\gamma_r - 1}(a) - \mu_a| \geq 1.7 \sqrt{\frac{\log\log(r-1) + 0.72\log(10.4/\alpha)}{r-1}}\right) \leq \alpha. $$
Therefore, we can bound the expectation $\sup_{P \in \Pcal(\mu_a)}\EE[|\hat\mu_{\gamma_r-1}(a) - \mu_a|]$ as follows:
\begin{align}
    \sup_{P \in \Pcal(\mu_a)}\EE[|\hat\mu_{\gamma_r-1}(a) - \mu_a|] &\leq \alpha + 1.7(1-\alpha)\sqrt{\frac{\log\log(r-1) + 0.72\log(10.4/\alpha)}{r-1}}\\
    &\leq \alpha + 1.7\sqrt{\frac{\log\log(r-1) + 0.72\log(10.4/\alpha)}{r-1}}
\end{align}
where in the last line, we remove the $(1-\alpha)$ term because $\alpha \in [0,1]$. Setting $\alpha = 1/\sqrt{r-1}$ for $r > 1$,  we obtain
\begin{align}
    \sup_{P \in \Pcal(\mu_a)}\EE[|\hat\mu_{\gamma_r-1}(a) - \mu_a|] &\leq \frac{1}{\sqrt{r-1}} + 1.7\sqrt{\frac{\log\log(r-1) + 0.72\log(10.4\sqrt{r-1})}{r-1}}\\
    &\leq \eta\sqrt\frac{\log(r-1)}{r-1}.
\end{align}
for some constant $\eta \in \RR^+$ that does not depend on $r$. We note that our bounds in the last line above are loose - these can be tightened significantly. By plugging this bound for $r \geq 2$ in Line \ref{line:summation_line}, we obtain our desired result:

\begin{align*}
    &\inf_{P \in \Pcal(\mu_a)}\frac{\EE[\log\left( E_t^{PrPl}(\Hcal_a^-, b) \right)]}{N_t(a)} - \left(\log\frac{1-\mu_a}{1-\xi} + \mu_a \log\frac{\mu_a(1-\xi)}{\xi(1-\mu_a)}\right) \\
    &\geq \frac{1}{N_t(a)} \frac{1}{b\xi(1-\xi)} \sum_{r=1}^{N_t(a)} -\sup_{P \in \Pcal(\mu_a)} \EE[|\hat\mu_{\gamma_r-1}(a) - \mu_a|]  \\
    &\geq  -\frac{1}{N_t(a)} \frac{1}{b\xi(1-\xi)} \left(1 + \sum_{r=2}^{N_t(a)} \eta\sqrt\frac{\log(r-1)}{{r-1}} \right) \\ 
    &\geq  -\frac{1}{N_t(a)} \frac{1}{b\xi(1-\xi)} \left(1 + \sum_{r=2}^{N_t(a)} \eta\sqrt\frac{\log(r-1)}{{r-1}} \right) \\ 
    & \geq  -\frac{1}{N_t(a)}\frac{1}{b\xi(1-\xi)} \left(1 + \eta'\sqrt{N_t(a)\log N_t(a)}  \right) = -O(\sqrt{\log N_t(a) / N_t(a) } )
\end{align*}
where $\eta' \in \RR^+$ is some positive constant that does not depend on $N_t(a)$. This now concludes our proof.

\subsection{Proof of Theorem \ref{thm:error_control_alg_2}}\label{proof:error_control_alg_2}

To provide this proof, we first use provide the following result from \cite{Ville1939}.

\begin{lemma}[Ville's Maximal Inequality (\citealp{Ville1939})]\label{lem:ville}
    For any non-negative martingale $L_t$ and any $x>1$, define a potentially infinite stopping time $N \coloneqq \inf\{t \geq 1: L_t \geq x\}$. Then, 
    $$\PP(\exists t: L_t \geq x) \leq \EE[L_0]/x. $$
\end{lemma}

Using Lemma \ref{lem:ville}, we show that the $2K$ sequential tests in Algorithm \ref{alg:i-GAI} provide $\delta$-level error control by being nonnegative super-martingales. We first show nonnegativity, and then show that our tests are supermartingales when the specified null is true (i.e. rejection of the null corresponding to the truth is at our desired level).

\paragraph{Nonnegatvity} We now show that any $E_t^{\text{PrPl}}(\Hcal_a^-, b), E_t^{\text{PrPl}}(\Hcal_a^-, b)$ is a nonnegative super-martingale. First, $\lambda_{t,a}^- \in [0, b/\xi]$, and $X_i - \xi \geq -\xi$, so $(1+\lambda_{t,a}^-(X_i -\xi)) \geq 1-b$ for all $i \in \NN$. Likewise, $\lambda_{t,a}^+ \in [0, -b/(1-\xi)]$ and $X_i - \xi \leq 1-\xi$, so $(1+\lambda_{t,a}^+(X_i - \xi)) \geq 1-b$ for all $i \in \NN$. Because $b \in (0,1)$, this implies that $E_t^{\text{PrPl}}(\Hcal_a^-, b), E_t^{\text{PrPl}}(\Hcal_a^+, b)$ is nonnegative for all $t \in \NN$. 

\paragraph{Supermartingale} We now establish that our processes $E_t^{\text{PrPl}}(\Hcal_a^-, b), E_t^{\text{PrPl}}(\Hcal_a^+, b)$ are nonnegative supermartingales if $\Hcal_a^-$ or $\Hcal_a^+$ is true, respectively. We show this result for $E_t^{\text{PrPl}}(\Hcal_a^-, b)$ first, assuming that $\Hcal_a^- = \{\Pcal(\mu_a): \mu_a \leq \xi\}$ is true. The proof for $E_t^{\text{PrPl}}(\Hcal_a^+, b)$ when $\Hcal_a^+ = \{\Pcal(\mu_a): \mu_a > \xi\}$ is true is symmetric. 

To see that $E_t^{\text{PrPl}}(\Hcal_a^-, b)$ is a supermartingale when the null $\Hcal_a^- = \{\Pcal(\mu_a): \mu_a \leq \xi\}$ is true, 

\begin{align}
    \EE[E_t^{\text{PrPl}}(\Hcal_a^-, b) | \Fcal_{t-1}] &= \left(\underbrace{\prod_{i < t: A_i = a} (1+\lambda_{i, a}^- (X_i - \xi))}_{E_{t-1}^{\text{PrPl}}(\Hcal_a^-, b)}\right) \times   \\
    & \left(\PP(A_t = a | \Fcal_{t-1}) \EE[1+ \lambda_{t,a}^-(X_i - \xi)]   + \PP(A_t \neq a | \Fcal_{t-1})  \right)\\
    &= E_{t-1}^{\text{PrPl}}(\Hcal_a^-, b) \times \left(\PP(A_t \neq a | \Fcal_{t-1}) +  \PP(A_t = a | \Fcal_{t-1})(1+ \underbrace{\lambda_{t,a}^-}_{(\geq 0 \text{ by defn.})}\underbrace{(\EE[X_i] - \xi)}_{(\leq 0 \text{ under $\Hcal_a^-$.})}\right)  \\
    &= E_{t-1}^{\text{PrPl}}(\Hcal_a^-, b) \times \left(\PP(A_t \neq a | \Fcal_{t-1}) +  \PP(A_t = a | \Fcal_{t-1})\underbrace{(1+ \lambda_{t,a}^-(\EE[X_i] - \xi)}_{\leq 1}\right)\\
    &\leq \EE_{t-1}^{\text{PrPl}}(\Hcal_a^-, b).
\end{align}  

Thus, when $\Hcal_a^-$ is true, $E_t^{\text{PrPl}}(\Hcal_a^-, b)$ is a supermartingale. The same argument holds for $E_t^{\text{PrPl}}(\Hcal_a^+, b)$ when $\Hcal_a^+$ is true. By direct applying Lemma \ref{lem:ville}, we obtain that:
\begin{align}
    &\PP\left(\exists t \in \NN \ : E_t^{\text{PrPl}}(\Hcal_a^-, b) \geq 2K/\delta \right) \leq \delta/(2K) \text{ when $\Hcal_a^-$ is true,} \\
    &\PP\left(\exists t \in \NN \ : E_t^{\text{PrPl}}(\Hcal_a^+, b) \geq 2K/\delta \right) \leq \delta/(2K) \text{ when $\Hcal_a^+$ is true.}
\end{align} 

By a simple union-bound argument, the probability of an error, i.e., rejecting $\Hcal^+_a$, $\Hcal^-_a$ when they are true for each $a \in [K]$, is controlled at our desired $\delta$-level:
\begin{align}
    &\PP(\exists t \in \NN \ s.t. \ \{\Gcal_t \not\subseteq \Gcal_{\bm\mu} \cup \{\Bcal_t \not\subseteq \Bcal_{\bm\mu}\}\} ) = \\
    &\PP\left(\exists t \in \NN, \ \left\{E_{t}^{\text{PrPl}}(\Hcal_a^-, b) > 2K/\delta \right\}_{a: \mu_a \leq \xi} \cup \left\{E_{t}^{\text{PrPl}}(\Hcal_a^+, b) > 2K/\delta \right\}_{a: \mu_a > \xi} \right) \leq \\
    & \sum_{a: \mu_a \leq \xi} \PP\left(\exists t \in \NN \ : E_{t}^{\text{PrPl}}(\Hcal_a^-, b) > 2K/\delta \right) +  \sum_{a: \mu_a > \xi} \PP\left(\exists t \in \NN \ : E_{t}^{\text{PrPl}}(\Hcal_a^+, b) > 2K/\delta \right) \leq \delta \\
\end{align}
Thus, $\PP(\exists t \in \NN \ s.t. \ \{\Gcal_t \not\subseteq \Gcal_{\bm\mu} \cup \{\Bcal_t \not\subseteq \Bcal_{\bm\mu}\}\} ) \leq \delta$, and we have our desired error control.

\subsection{Proof of Theorem \ref{theorem:minimax_lower_bounds_stopping_time}} \label{proof:thm_stopping_times}

Theorem \ref{theorem:minimax_lower_bounds_stopping_time} provides an upper bound on the asymptotic stopping time as $\delta \rightarrow 0$ for both the THR and GAI problems. For this section, without loss of generality, we assume that $\mu_1 > \mu_2 > \mu_3 .... > \mu_K$, and that there exists at least one arm that is above the threshold value $\xi$.  We provide a proof on the expected stopping time of $\tau_{\Gcal, 1}$. To obtain the results provided in our theorem, we repeat this argument for $\tau_{\Gcal, 2}$, ..., $\tau_{\Gcal, G}$, where $G$ is the number of good arms. We provide a remark on our analysis, which requires a small modification of Algorithm \ref{alg:i-GAI} in Section \ref{sec_5:sampling_scheme}. Before starting our proof, we first provide necessary results for analysis (including a proof of our test being of power one, i.e. labels arms in finite time), below. 

\subsubsection{Lemmas for Theorem \ref{theorem:minimax_lower_bounds_stopping_time}.}
We provide two lemmas regarding the upper/lower bounds on the stopping time and a technical lemma for analyzing joint probabilities of the stopping time and arm selection. To begin, we provide lower bounds on the stopping time $\tau$ based on our $e$-process. 

\begin{lemma}[Lower bound on $\tau$.]\label{lem:tau_lower_bound}
    For any $c \in [0, \infty)$, $b \in (0,1)$, $\xi \in (0,1)$, $\bm\mu \in [0,1]^K$,   $P \in \Pcal(\mu)$, and $\Fcal_{t-1}$-measurable sampling scheme $\pi$. At time  $\tau_a = \inf\{t \in \NN : \max\left(E_t^{\text{PrPl}}(\Hcal_a^+, b), E_t^{\text{PrPl}}(\Hcal_a^-, b)\right) \geq c\}$, the number of pulls of arm $a$ at stopping time $\tau_a$, $N_{\tau_a}(a)$, is upper bounded almost surely as follows:
    $$N_{\tau_a}(a) \geq \frac{\log(c)}{\log(1 + b\max\left(\frac{\xi}{1-\xi}, \frac{(1-\xi)}{\xi})\right)}.$$ 
\end{lemma}

\begin{proof}[Proof of Lemma \ref{lem:tau_lower_bound}]
    We first rewrite $E_t^{\text{PrPl}}(\Hcal_a^+, b)$ and $E_t^{\text{PrPl}}(\Hcal_a^-, b)$ as defined in Definition \ref{defn:predicable_maximin}:
    \begin{align}
        E_{\tau_a}^{\text{PrPl}}(\Hcal_a^+, b) &=  \sum_{t=1}^{\tau_a} \mathbf{1}[A_i = a] \log(1+\lambda_{t,a}^+(X_i - \xi) ) \\ 
        E_{\tau_a}^{\text{PrPl}}(\Hcal_a^-, b) &= \sum_{t=1}^{\tau_a} \mathbf{1}[A_i = a] \log(1+\lambda_{t,a}^-(X_i - \xi) )
    \end{align}

    For each term in $E_t^{\text{PrPl}}(\Hcal_a^+, b)$, $E_t^{\text{PrPl}}(\Hcal_a^-, b)$, note that $\log(1+\lambda_{t,a}^+(X_i - \xi)) \leq \log\left(1 + b\max(\frac{\xi}{1-\xi}, \frac{1-\xi}{\xi})\right)$, and therefore, $\max(E_{\tau_a}^{\text{PrPl}}(\Hcal_a^+, b), E_{\tau_a}^{\text{PrPl}}(\Hcal_a^-, b)) \leq N_{\tau_a}(a)\log\left(1 + b\max(\frac{\xi}{1-\xi}, \frac{1-\xi}{\xi})\right) $ for any $\tau_a \in \NN$. By definition, for any $\tau_a$, we satisfy $\max(E_{\tau_a}^{\text{PrPl}}(\Hcal_a^+, b), E_{\tau_a}^{\text{PrPl}}(\Hcal_a^-, b)) \geq c$, which implies our desired result:
    \begin{align}
        &N_{\tau_a}(a)\log\left(1 + b\max\left(\frac{\xi}{1-\xi}, \frac{1-\xi}{\xi}\right)\right) \geq \max\left(E_{\tau_a}^{\text{PrPl}}(\Hcal_a^+, b), E_{\tau_a}^{\text{PrPl}}(\Hcal_a^-, b)\right)  \geq c \\
        &\implies N_{\tau_a}(a) \geq \frac{c}{\log\left(1 + b\max\left(\frac{\xi}{1-\xi}, \frac{1-\xi}{\xi}\right)\right)}. 
    \end{align}
\end{proof}

We now show that our stopping times are almost surely finite in Lemma \ref{lem:test_power_1}.

\begin{lemma}[Finite Stopping Times]\label{lem:test_power_1}
Assume that $\mu_a \neq \xi$, and $N_t(a) \rightarrow \infty$ almost surely as $t \rightarrow \infty$. Then, $\max(E_t^{\text{PrPl}}(\Hcal_a^+, b), E_t^{\text{PrPl}}(\Hcal_a^-, b)) \rightarrow \infty$ as $t \rightarrow \infty$ almost surely, and subsequently $\PP(\tau_a < \infty) = 1$. 
\end{lemma}

\begin{proof}[Proof of Lemma \ref{lem:test_power_1}]

We provide a proof sketch, then provide each step below.

\paragraph{Outline.} We first show that when we have a single arm, i.e., $K = 1$ with fixed sampling scheme that samples the single arm, one of our $e$-process among $E_t(\Hcal_a^+, b)$ $E_t(\Hcal_a^- , b)$ diverges to infinity almost surely, and thus $\tau_a = \inf\{t \in \NN: \max(E_t(\Hcal_a^+, b), E_t(\Hcal_a^- , b)) \geq 2K/\delta \}$ is finite almost surely. Then, using a result from \cite{fact_e1}, we generalize our result to show that all stopping times $\tau_a$ for $a \in [K]$, $K \geq 1$ are finite almost surely.

\paragraph{Step 1: Finite Stopping Times Almost Surely for a Single Arm.} We consider the case where $\Hcal_a^+$ is true; the proof when $\Hcal_a^-$ is true is symmetric and follows the same arguments. We first show that for the test martingale using oracle knowledge, i.e., $\lambda_{a}^- = \frac{\mu_a - \xi}{\xi(1-\xi)}$, the worst-case $e$-power is strictly positive. Note that under our assumptions that $\mu_a \in [\xi(1-b), b(-\xi) + \xi]$, it implies that $\lambda_{a}^- \in [-\frac{b}{1-\xi}, \frac{b}{\xi}]$, so we do not need to threshold its value at $b/\xi$.  
\begin{align}
    \inf_{P \in \Pcal(\mu_a)} \frac{\sum_{i=1}^t\EE[\log(1+\lambda_{a}^- (X_i - \xi))]}{t}  &= \frac{\sum_{i=1}^t\EE_{X_i \sim \text{Bern}(\mu_a)}[\log(1+\lambda_{a}^- (X_i - \xi))]}{t} \label{line:kl_diff} \\
    & = \EE_{X \sim \text{Bern}(\mu_a)}[\log\left(1 + \lambda_{a}^-(X_i - \xi)\right)] \\ 
    & = \mu_a\log\frac{\mu_a}{\xi} + (1-\mu_a)\log\frac{1-\mu_a}{1-\xi} \\
    & = d(\mu_a, \xi),
\end{align}
where $d(\mu_a, \xi)$ is the KL-divergence between two Bernoulli distributions with mean $\mu_a$ and $\xi$. Note that our reduction in line \ref{line:kl_diff} is a direct application of Lemma \ref{lem:worst_case_instance}. Under our assumptions that $\mu_a \neq \xi$ for all $\mu_a$, $d(\mu_a, \xi)$ is strictly larger than 0. We now construct a function $f(\epsilon)$, which is equivalent to line \ref{line:kl_diff} when $\epsilon_1 = 0$:
\begin{align}
    f(\epsilon) &= \EE_{X_i \sim \text{Bern}(\mu_a)}\left[\mathbf{1}[X_i \geq \xi] \log\left(1 + \frac{(\mu_a - \epsilon) - \xi}{\xi(1-\xi)}(X_i - \xi) \right)\right] +  \\
    &\quad\quad \EE_{X_i \sim \text{Bern}(\mu_a)}\left[\mathbf{1}[X_i < \xi] \log\left(1 + \frac{(\mu_a + \epsilon) - \xi}{\xi(1-\xi)}(X_i - \xi) \right)\right].
\end{align}
Note that $f(\epsilon) = d(\mu_a, \xi)$ when $\epsilon_1 = 0$, and because $f(\epsilon_1)$ is continuous in $\epsilon$, there exists an $\epsilon_1$ such that $f(\epsilon_1) = d(\mu_a , \xi)/p$ where $p \in (1, \infty)$ is some fixed constant. We now use two applications of the strong law of large numbers to obtain that $E_{t}(\Hcal_1^-, b)$ diverges to infinity almost surely, which directly implies that $\tau_1$ is finite. 

We consider the set of sample path $\omega \in \Omega$, such that $\PP(\Omega) = 1$. For each random variable, we index by the sample path $\omega$, i.e., $\hat\mu_{t-1}(a, \omega)$, $X_{t}(\omega), E_t(\Hcal_a^-, b, \omega)$, and $\lambda_{t,a}^-(\omega)$. By the strong law of large numbers, we know that for all $\omega \in \Omega$, there exists $t_1(\omega) < \infty$ such that guarantees the following for all $t > t_1(\omega)$, $|\hat\mu_{t-1}(a, \omega) - \mu_a|  < \epsilon_1$. We write the log of our $e$-process $E_t(\Hcal_a^-, b)$ as follows:
\begin{align}
    \log(E_t(\Hcal_a^-, b, \omega)) = \underbrace{\sum_{i=1}^{t_1(\omega)} \log( 1 + \lambda_{t,a}^-(\omega) (X_i(\omega) - \xi) )}_{(a)} + \underbrace{\sum_{i=t_1(\omega)+ 1}^{t} \log( 1 + \lambda_{t,a}^-(\omega) (X_i(\omega) - \xi) )}_{(b)} 
\end{align}

Note that summation $(a)$ is finite due to $t_1(\omega) < \infty$, and each term within $(a)$ being bounded. We now show that summation $(b)$ diverges to infinity below. We first rewrite $(b)$ as follows:
\begin{align}
    (b) &= \sum_{i=t_1(\omega)+ 1}^{t} \mathbf{1}[X_i \geq \xi]\log\left( 1 + \frac{\hat\mu_{t-1}(a,\omega) - \xi}{\xi(1-\xi)}(X_i(\omega) - \xi) \right) \\ 
    &\quad\quad + \sum_{i=t_1(\omega)+ 1}^{t} \mathbf{1}[X_i < \xi] \log\left( 1 + \frac{\hat\mu_{t-1}(a,\omega) - \xi}{\xi(1-\xi)} (X_i(\omega) - \xi) \right).
\end{align}
Recall that for any $i > t_1(\omega)$, $|\hat\mu_{t-1}(a, \omega) - \mu_a| < \epsilon_1$. Thus, we can compare $(b)$ to a strictly smaller term for any realization  $\left(X_i(\omega)\right)_{i \in \NN, \omega \in \Omega}$: 
\begin{align}
    (b) \geq & \sum_{i=t_1(\omega)+ 1}^{t} \mathbf{1}[X_i \geq \xi]\log\left( 1 + \frac{\mu_a - \epsilon_1 - \xi}{\xi(1-\xi)}(X_i(\omega) - \xi) \right) \label{line:ineq_power_1}\\ 
    &\quad\quad + \sum_{i=t_1(\omega)+ 1}^{t} \mathbf{1}[X_i < \xi] \log\left( 1 + \frac{\mu_a + \epsilon_1 - \xi}{\xi(1-\xi)} (X_i(\omega) - \xi) \right).\label{line:ineq_power_2}
\end{align}
By another application of the strong law of large numbers, the RHS converges to a constant as $t \rightarrow \infty$ for all $\omega \in \Omega$, i.e. 
\begin{align}
    &\frac{\sum_{i=t_1(\omega)+ 1}^{t} \mathbf{1}[X_i \geq \xi]\log\left( 1 + \frac{\mu_a - \epsilon_1 - \xi}{\xi(1-\xi)}(X_i(\omega) - \xi) \right)+ \sum_{i=t_1(\omega)+ 1}^{t} \mathbf{1}[X_i < \xi] \log\left( 1 + \frac{\mu_a + \epsilon_1 - \xi}{\xi(1-\xi)} (X_i(\omega) - \xi) \right)}{t - t_1(\omega) - 1}\\
    &\rightarrow f(\epsilon_1) = d(\mu_a, \xi)/p > 0.
\end{align} 
This implies that the non-normalized term diverges to infinity, and by lines \ref{line:ineq_power_1} and \ref{line:ineq_power_2}, this implies that $(b)$ also must diverge to infinity for all $\omega \in \Omega$. Because $\log(E_t(\Hcal_a^-, b, \omega)) = \underbrace{(a)}_{\text{finite}} + \underbrace{(b)}_{\rightarrow \infty}$ for all $\omega \in \Omega$, $E_t(\Hcal_a^-, b)$ also diverges to infinity almost surely, and therefore $\tau_a = \inf\{t \in \NN: \max(E_t(\Hcal_a^+, b), E_t(\Hcal_a^- , b)) \geq 2K/\delta \}$ is finite almost surely. This concludes the proof when $\Hcal_a^+$ is true.
By the same argument, when $\Hcal_a^-$ is true, $E_t(\Hcal_a^+, b)$ diverges to infinity almost surely, and thus $\tau_a$ is finite.

\paragraph{Step 2: Generalizing to Multiple Arms.}

We generalize our result to multiple arms under our sampling scheme by using Lemma \ref{lem:convergence_bandit} below.

\begin{lemma}[Fact E.1, \citealp{fact_e1}]\label{lem:convergence_bandit}
        Suppose that $Y_n \rightarrow Y$ a.s. as $n \rightarrow \infty$, and $N(t) \rightarrow \infty$ a.s. as $t \rightarrow \infty$. Then $Y_{N(t)} \rightarrow Y$ a.s. as $t \rightarrow \infty$.
\end{lemma}

Note that as $t \rightarrow \infty$, there exists at least one arm $a_1^*$ such that $N_t(a_1^*) \rightarrow \infty$. Note that as $N_t(a_1^*) \rightarrow \infty$, $E_t(\Hcal_{a_1^*}^-, b) \rightarrow \infty$ almost surely, so $\tau_{a_1^*}$ is finite almost surely. Note that after $\tau_{a_1^*} < \infty$, we stop sampling arm $a_1^*$, so there must now be a new arm $a_2^*$ such that $N_t(a_2^*) \rightarrow \infty$ almost surely. By the same argument, $\tau_{a_2^*} < \infty$ almost surely. We repeat this argument until all arms are labeled, resulting in $\tau_a < \infty$ for all $a \in [K]$.

\end{proof}

Finally, we introduce one final lemma that enables our analysis of the stopping time. This provides the relationship between the joint and product probabilities of events $\mathbf{1}[\tau_a \geq t]$ and $\mathbf{1}[A_t \neq a]$. 

\begin{lemma}[Joint Probabilities are less than Product of Probabilities.]\label{lem:product_greater_joint}
For all $t \in \NN$, 
$$\PP(\tau_a \geq t, A_t \neq a) \leq \PP(\tau_a \geq t)\PP(A_t \neq a).$$ 
\end{lemma}

\begin{proof}[Proof of Lemma \ref{lem:product_greater_joint}]

The proof simply consists of algebraic manipulation, and the fact that $\PP(\tau_a < t, A_t \neq a) = \PP(\tau_a < t)$ under our sampling scheme (i.e., after an arm $a$ is labeled, it is not sampled again). 
\begin{align}
    &\PP(\tau_a \geq t, A_t \neq a) - \PP(\tau_a \geq t)\PP(A_t \neq a) = \\
    &\PP(\tau_a \geq t, A_t \neq a) - \PP(\tau_a \geq t) \left( \PP(A_t \neq a, \tau_a \geq t) + \PP(A_t \neq a, \tau_a < t) \right) = \\
    &\PP(\tau_a \geq t, A_t \neq a) - \PP(\tau_a \geq t)\left( \PP(A_t \neq a, \tau_a \geq t) + \PP(\tau_a < t) \right) =  \\ 
    &\PP(\tau_a \geq t, A_t \neq a) - \PP(\tau_a \geq t)\left( 1- \PP(\tau_a \geq t) + \PP(A_t \neq a, \tau_a \geq t) \right) =\\ 
    & \PP(\tau_a \geq t, A_t \neq a) - \PP\left(\tau_a \geq t)(1 - \PP(\tau_a \geq t, A_t \neq a) - \PP(\tau_a \geq t, A_t = a) + \PP(A_t \neq a, \tau_a \geq t) \right) = \\
    &\PP(\tau_a \geq t, A_t \neq a) - \PP(\tau_a \geq t)\left( 1 - \PP(\tau_a \geq t, A_t = a) \right) = \\
    & \PP(\tau_a \geq t, A_t \neq a) - \PP(\tau_a \geq t) + \PP(\tau_a \geq t)\PP(\tau_a \geq t, A_t = a) = \\
    & \PP(\tau_a \geq t) \PP(\tau_a \geq t, A_t = a)- \PP(\tau_a \geq t, A_t = a) = \\
    & (1- \PP(\tau_a \geq t))\PP(\tau_a \geq t, A_t =a) \leq 0
\end{align}
where the last line simply follows from the fact that $\PP(\tau_a \geq t) \leq 1$. Note that this proof also holds for $\tau_a^+ = \inf\{t \in \NN: E_t(\Hcal_a^+, b) \geq 2K/\delta\}$ (i.e., rejecting that $\mu_a >\xi$) or  $\tau_a^- = \inf\{t \in \NN: E_t(\Hcal_a^-, b) \geq 2K/\delta\}$ (i.e., rejecting that $\mu_a \leq\xi$).
\end{proof}

\subsubsection{Proof of Minimax Lower Bounds in Theorem \ref{theorem:minimax_lower_bounds_stopping_time}}

The asymptotic minimax lower bounds we provide on the stopping time directly come from \cite{kano2019good} and \cite{degenne2019pure}, who provide asymptotic bounds for the parametric case.  
\begin{theorem}[Asymptotic Lower Bounds for Bernoulli Good Arm Identification, \citealt{kano2019good}] Assume that each arm $a$ has a stationary Bernoulli distribution. Let there exist $G$ good arms, i.e. $\mu_1 > ... \mu_G > \xi$. Then, for any $i \leq G$, the asymptotic expected stopping time $\tau_{\Gcal, i}$ for any $\delta$-correct algorithm is lower bounded as follows:
\begin{align}
    \lim_{\delta \rightarrow 0} \inf_{(\pi, \tau)}\frac{\EE[\tau_{\Gcal, i}]}{\log(1/\delta)} \geq \sum_{j=1}^i \frac{1}{d(\mu_j, \xi)}.
\end{align}
\end{theorem}

We obtain this result by taking the limit with respect to $\delta$ using Theorem 1 of \cite{kano2019good}. For $\tau_{\text{stop}}$, i.e., the stopping time for labeling all arms, we directly cite Theorem 1 of \cite{degenne2019pure}.
\begin{theorem}[Asymptotic Lower Bounds for $\tau_{\text{stop}}$, \citealt{degenne2019pure}]
    Assume that each arm $a$ has a stationary Bernoulli distribution, and $\mu_a \neq \xi$ for all $a \in [K]$. Then, for any $\delta$-correct algorithm $(\pi, \tau)$, 
    \begin{align}
        \lim_{\delta \rightarrow 0} \inf_{(\pi, \tau)}\frac{\EE[\tau_{\text{stop}}]}{\log(1/\delta)} \geq \sum_{i=1}^K \frac{1}{d(\mu_i, \xi)}.
    \end{align}
\end{theorem}

Note that because stationary Bernoulli arms $P$ are an instance in $\Pcal(\bm\mu)$, we immediately get our minimax-lower bounds. Let $P'$ denote the stationary Bernoulli instance with mean vector $\bm\mu$. Because the infima of supremums is greater than the suprema of infimums,
\begin{align}
    \inf_{(\pi, \tau)}\sup_{P \in \Pcal(\bm\mu)}\frac{\EE[\tau_{\Gcal, i}]}{\log(1/\delta)} \geq \sup_{P \in \Pcal(\bm\mu)}\inf_{(\pi, \tau)}\frac{\EE[\tau_{\Gcal, i}]} {\log(1/\delta)} \geq \inf_{(\pi, \tau)}\frac{\EE_{P'}[\tau_{\Gcal, i}]}{\log(1/\delta)} \label{line:repeat}
\end{align}
By taking limits, we get the lower bound in Theorem \ref{theorem:minimax_lower_bounds_stopping_time}. The same argument in line \ref{line:repeat} applies for $\tau_{\text{stop}}$ to get the asymptotic lower bound presented in Theorem \ref{theorem:minimax_lower_bounds_stopping_time}.

\subsubsection{Proof of Achieving Minimax Lower Bounds in Theorem \ref{theorem:minimax_lower_bounds_stopping_time}.} We now provide our proof for achieving the minimax lower bounds shown above. To conduct our analysis, we partition the total number of samples into the number of samples between the time of identifying successive (good) arms to obtain our results. Our proof focuses on the case of identifying \emph{one} good arm first. We then generalize our results to successively identifying good arms (or labeling arms as good or bad), under a slight modification of Algorithm \ref{alg:i-GAI} that discards samples and resets the $e$-processes after an arm has been identified.

\paragraph{Proof Sketch.} Our proof proceeds in four main steps. The first three steps are focused on the proof of optimality for the first stopping time $\tau_{\Gcal, 1}$, under the assumption that at least one good arm exists. In the first step, we leverage that expectation of one of our $e$-processes, $E_t(\Hcal_a^-, b), E_t(\Hcal_a^+, b)$, are close to the threshold value of $\frac{2K}{\delta}$. We then apply the results of Theorem \ref{thm:minmax_opt_gen_bern} to obtain a bound in terms of the expected number of arm pulls, $N_t(a)$. In the second step, we provide a lower bound on our term with the expectation of $N_t(a)$ in terms of the expected stopping time by using lemmas \ref{lem:tau_lower_bound}, \ref{lem:test_power_1}, and \ref{lem:product_greater_joint}. In Step 3, we conduct an asymptotic analysis to show that we achieve the desired result for $\tau_{\Gcal, 1}$. In our final step, we consider the number of samples required between arm pulls, and show how our results for $\tau_{\Gcal, 1}$ generalize to obtain our desired results for $\tau_{\Gcal, i}$ and $\tau_{\text{stop}}$. 



\paragraph{Step 1: Constructing Bounds in Terms of $N_{\tau_1^-}(1)$.} We consider the stopping time of identifying one arm above the threshold, $\tau_{\Gcal, 1}$. By definition, $\tau_{\Gcal, 1} \leq \tau^-_1$, where $\tau_1^- = \inf\{t \in \NN: E_t^{\text{PrPl}}(\Hcal_1^-, b) \geq \log(2K/\delta) \}$, i.e. the stopping time for rejecting null hypothesis set $\Hcal_1^-: \{\Pcal(\bm\mu):\mu_1 \leq \xi\}$. At time $\tau_1^-$, the following must hold:
\begin{align}
    &\EE\left[\sum_{t \leq \tau_1^-: A_t = 1}\log( 1+\lambda_{t,1}^-(X_t - \xi) )\right] \geq \log(2K/\delta), \\
    &\EE\left[\sum_{t \leq \tau_1^-: A_t = 1}\log( 1+\lambda_{t,1}^-(X_t - \xi) )\right] \leq \log(2K/\delta) + \log\left(1 + b\max\left(\frac{\xi}{1-\xi}, \frac{1-\xi}{\xi}\right)\right) \label{line:upper_bound}
\end{align}

The first statement is true by definition of $\tau_1^-$. The second follows from the fact that $\log(1 + \lambda_{t,1}^-(X_t - \xi))$ for $t=\tau_1^-$ must be less than $\log\left(1 + b\max\left(\frac{\xi}{1-\xi}, \frac{1-\xi}{\xi}\right)\right)$, giving us an upper bound on the value of the $e$-process $E_t^{\text{PrPl}}(\Hcal_a^-, b)$  at stopping time $t = \tau_1^-$. We now provide a lower bound on the value of $\EE\left[\sum_{t \leq \tau_1^-: A_t = 1}\log( 1+\lambda_{t,1}^-(X_t - \xi) )\right]$ in terms of $N_{\tau_1^-}(1)$:

\begin{align}
    \EE\left[\sum_{t \leq \tau_1^-: A_t = 1}\log( 1+\lambda_{t,1}^-(X_t - \xi) )\right] & = \EE\left[ N_{\tau_1^-}(1) \frac{\sum_{t: A_t = 1}^{\tau_1^-
} \log(1+ \lambda_{t,1}^-(X_t - \xi )}{N_{\tau_1^-}(1)} \right] \\
    &=\EE \left[ N_{\tau_1^-}(1) \ \EE \left[\frac{1}{n} \sum_{t: A_t = 1}^{\tau_1^-}  \log(1+ \lambda_{t,1}^-(X_t - \xi )) \ \big{|} \ N_{\tau_1^-}(1) = n \right ] \right] \\
    &\geq \EE \left[ N_{\tau_1^-}(1) \ \EE \left[ \log\frac{1-\mu_1}{1-\xi} + \mu_1\log\frac{\mu_1(1-\xi)}{\xi(1-\mu_1)} - \eta'\sqrt{\frac{\log n}{n}} \ \Bigg{|} \ N_{\tau_1^-}(1) = n \right ] \right] \label{line:lower_bdd}  \\
    &= \left(\log\frac{1-\mu_1}{1-\xi} + \mu_1\log\frac{\mu_1(1-\xi)}{\xi(1-\mu_1)}\right)\EE\left[N_{\tau_1^-}(1)\right] \\
    &\quad - \eta'\EE\left[N_{\tau_1^-}(1) \EE\left[\sqrt{\frac{\log n}{n }} \ \Bigg{|} \  N_{\tau_1^-}(1) = n\right]\right] \label{line:tousejensen}
\end{align}
where $\eta' \in \RR^+$ is a constant that does not depend on $N_{\tau_1^-}(1)$, and line \ref{line:lower_bdd} is a direct application of Theorem \ref{thm:minmax_opt_gen_bern}. We then note that the function $f(n) = \sqrt{\frac{\log n}{n}}$ is a concave function for all $n \geq 54$, which can be checked with a simple second-derivative test:
$$\frac{\partial^2}{\partial^2 n} f(n) = \frac{\left(\frac{\log n}{n}\right)^{3/2}(3- \log^2(n)) }{4\log^4(n)},$$
By Lemma \ref{lem:tau_lower_bound}, $N_{\tau_1^-}(1) \geq \frac{\log(2K/\delta)}{\log\left(1+ b\max\left(\frac{\xi}{1-\xi}, \frac{1-\xi}{\xi}\right)\right)}$ almost surely, and therefore for  small enough $\delta$, $\sqrt{\frac{\log N_{\tau_1^-}(1)}{N_{\tau_1^-}(1)}}$ is concave almost surely. Using concavity, we apply Jensen's inequality to the term on line \ref{line:tousejensen} to obtain our desired lower bound:

\begin{align}
    \EE&\left[\sum_{t \leq \tau_1^-: A_t = 1}\log( 1+\lambda_{t,1}^-(X_t - \xi) )\right] = \left(\log\frac{1-\mu_1}{1-\xi} + \mu_1\log\frac{\mu_1(1-\xi)}{\xi(1-\mu_1)}\right)\EE\left[N_{\tau_1^-}(1)\right] \\
    &\qquad - \eta'\EE\left[N_{\tau_1^-}(1) \EE\left[\sqrt{\frac{\log n}{n }} \ \Bigg{|} \  N_{\tau_1^-(1)} = n\right]\right] \\ 
    &\quad \geq \left(\log\frac{1-\mu_1}{1-\xi} + \mu_1\log\frac{\mu_1(1-\xi)}{\xi(1-\mu_1)}\right)\EE\left[N_{\tau_1^-}(1)\right]  \\
     &\qquad - \eta'\EE\left[N_{\tau_1^-}(1) \sqrt{\frac{\log \EE\left[N_{\tau_1^-}(1)\right]}{\EE\left[N_{\tau_1^-}(1)\right]}}\right] \\
     &\quad=  \left(\log\frac{1-\mu_1}{1-\xi} + \mu_1\log\frac{\mu_1(1-\xi)}{\xi(1-\mu_1)}\right)\EE\left[N_{\tau_1^-}(1)\right] - \eta'\sqrt{\EE\left[N_{\tau_1^-}(1)\right]\log \EE\left[N_{\tau_1^-}(1)\right]}.
\end{align}

By combining our lower bound with the upper bound shown in line \ref{line:upper_bound}, we obtain the inequality
\begin{align}
    &\EE\left[N_{\tau_1^-}(1)\right]\left(\left(\log\frac{1-\mu_1}{1-\xi} + \mu_1\log\frac{\mu_1(1-\xi)}{\xi(1-\mu_1)}\right) - \eta'\sqrt{\frac{\log \EE\left[N_{\tau_1^-}(1)\right]}{\EE\left[N_{\tau_1^-}(1)\right]}}\right)  \leq \\
    &\quad \log(2K/\delta) + \log\left(1 + b\max\left(\frac{\xi}{1-\xi}, \frac{1-\xi}{\xi}\right)\right). \label{lines:final_steps}
\end{align}

\paragraph{Step 2: Bounding Arm Pulls as a Function of Stopping Time.} We now focus on the two main terms on the left side of our inequality, terms $(a)$ and $(b)$:

$$\underbrace{\EE\left[N_{\tau_1^-}(1)\right]}_{(a)}\underbrace{\left(\left(\log\frac{1-\mu_1}{1-\xi} + \mu_1\log\frac{\mu_1(1-\xi)}{\xi(1-\mu_1)}\right) - \eta'\sqrt{\frac{\log \EE\left[N_{\tau_1^-}(1)\right]}{\EE\left[N_{\tau_1^-}(1)]\right]}}\right)}_{(b)} $$

We now seek to lower bound terms $(a)$ and $(b)$, which in turn lower bounds our total expression. We begin by providing an lower bound for term $(b)$.

The only term we have to bound in $(b)$ is the term $\sqrt{\frac{\log \EE\left[N_{\tau_1^-}(1)\right]}{\EE\left[N_{\tau_1^-}(1)\right]}}$. Note that the function $f(n) = \sqrt{\frac{\log n}{n}}$ is a monotonically decreasing function for $n > 10$. By Lemma \ref{lem:tau_lower_bound}, $N_{\tau_1^-}(1) \geq \frac{\log(2K/\delta)}{\log\left(1+ b\max\left(\frac{\xi}{1-\xi}, \frac{1-\xi}{\xi}\right)\right)}$ almost surely and therefore $\EE\left[N_{\tau_1^-}(1)\right]$ is lower bounded by $\frac{\log(2K/\delta)}{\log\left(1+ b\max\left(\frac{\xi}{1-\xi}, \frac{1-\xi}{\xi}\right)\right)}$ as well. Therefore, for small enough $\delta$, $\sqrt{\frac{\log \EE\left[N_{\tau_1^-}(1)\right]}{\EE \left[N_{\tau_1^-}(1)\right]}}$ is monotonically decreasing with respect to $\EE\left[N_{\tau_1^-}(1)\right]$. Thus, we seek to provide a \emph{lower bound} on $\EE\left[N_{\tau_1^-}(1)\right]$ to obtain a lower bound for expression $(b)$. 

We provide this bound by using Lemma \ref{lem:tau_lower_bound}, which provides an almost sure lower bound as a function of $\log(1/\delta)$. Because $m(\delta) = \frac{\log(2K/\delta)}{\log\left(1 + b\max\left(\frac{\xi}{1-\xi}, \frac{1-\xi}{\xi}\right)\right)}$ is an almost sure lower bound on $N_{\tau_1^-}(1)$, it is also an almost sure lower bound on $\tau_1^- \geq N_{\tau_1^-}(1)$. We first lower bound $\EE\left[N_{\tau_1^-}(1)\right]$ by $\EE\left[N_{m(\delta)}(1)\right]$:

\begin{align}
    \EE[N_{\tau_1^-}(1)] &= \EE\left[\sum_{t=1}^{\tau_1^-} \mathbf{1}[A_t = 1]\right]\\
    &= \EE\left[\sum_{t=1}^{m(\delta)} \mathbf{1}[A_t = 1]\right] + \EE\left[\sum_{t=m(\delta)+1}^{\tau_1^-} \mathbf{1}[A_t = 1]\right]\\
    & \geq \EE\left[\sum_{t=1}^{m(\delta)} \mathbf{1}[A_t = 1]\right] = \EE[N_{m(\delta)}(1)].
\end{align}

To proceed further with our lower bounds, we now turn to the conditions presented in Definition \ref{defn:regret_minimizing_scheme}, which states that there exists a $c \in \RR^+$ not dependent on $t$ such that $\PP(A_t \neq 1) \leq c/t$. Then, number of sub-optimal arm pulls, i.e. $\EE[N_t(a)] = \EE[\sum_{i=1}^t\mathbf{1}[A_t \neq 1]] \leq c(\log(t)+1)$ for all $a \neq 1.$ Because $m(\delta)$ is a fixed sample size (does not depend on data realization),  
\begin{align}
    \EE[N_{m(\delta)}(1)] &\geq m(\delta) - c\log\left(m(\delta)\right) - c\\
    &= \frac{\log(2K/\delta)}{\log(1 + b\max\left(\frac{\xi}{1-\xi}, \frac{(1-\xi)}{\xi})\right)} - c\log\left(\frac{\log(2K/\delta)}{\log(1 + b\max\left(\frac{\xi}{1-\xi}, \frac{(1-\xi)}{\xi})\right)}\right) - c.
\end{align}

Plugging this expression back into term $(b)$, we obtain the following, where $d(\mu_1, \xi) = \log\frac{1-\mu_1}{1-\xi} + \mu_1 \log\frac{\mu_1(1-\xi)}{\xi(1-\mu_1)}$, 
\begin{align}
    (b) &\geq d(\mu_1, \xi) - \eta'\sqrt{\frac{\log \EE[N_{\tau_1^-}(1)]}{\EE[N_{\tau_1^-}(1)]}} \\
    &\geq d(\mu_1, \xi) - \eta'\sqrt{\frac{\log(f(\delta))}{f(\delta)}}.
\end{align}
where $f(\delta) = \frac{\log(2K/\delta)}{\log\left(1 + b\max\left(\frac{\xi}{1-\xi}, \frac{1-\xi}{\xi}\right)\right)} - c\log\left(\frac{\log(2K/\delta)}{\log\left(1 + b\max\left(\frac{\xi}{1-\xi}, \frac{1-\xi}{\xi}\right)\right)}\right) - c$. Note that as $\delta \rightarrow 0$, $f(\delta) \rightarrow \infty$ due to $c, \eta' \in \RR^+$ being a constant value that does not depend on $\delta$ (equivalently $t$). 

We now turn to bounding term $(a)$. To construct our bounds in terms of $\EE[\tau_1^-]$, we leverage the lower bound on $\tau_1^- \geq N_{\tau_1^-}(1)$ given in Lemma \ref{lem:tau_lower_bound} to obtain a range of $t$ where the probability of stopping is zero. We then consider $t$ where we have nonzero probability of stopping, and apply Lemma \ref{lem:product_greater_joint} to obtain our results. We first start by re-expressing $(a)$ in terms of $\EE[\tau_1^-]$. 

\begin{align}
    (a) &= \EE\left[\tau_1^{-} - \sum_{a\neq 1} N_{\tau_1^-}(a)\right] \\
    &= \EE[\tau_1^-] - \EE\left[\sum_{t=1}^{\tau_1^-} \mathbf{1}[A_t \neq 1]\right] \\
    &= \EE[\tau_1^-] - \EE\left[\sum_{t=1}^{m(\delta)} \mathbf{1}[A_t \neq 1] + \sum_{t=1 + m(\delta)}^{\tau_1^-} \mathbf{1}[A_t \neq 1]\right]\\
    &\geq \EE[\tau_1^-] - \left( c \sum_{t=1}^{m(\delta)} \frac{1}{t} \right) - \EE\left[\sum_{t=1+m(\delta)}^{\tau_1^-}\mathbf{1}[A_t \neq 1]\right]\\
    &\geq \EE[\tau_1^-] - c\left(\log(m(\delta)) + 1 \right) - \underbrace{\EE\left[\sum_{t=1+m(\delta)}^{\tau_1^-}\mathbf{1}[A_t \neq 1]\right]}_{(c)}.
\end{align}

We provide an upper bound for the term $(c)$. First, we can rewrite term $(c)$, which is finite due to $(c) \leq \EE[\tau_1^-] < \infty$ (proven in Lemma \ref{lem:test_power_1}). This allows us to rearrange sums within the expectation, resulting in the following:
\begin{align}
    (c) &= 
    \EE\left[\sum_{T = 1 + m(\delta)}^\infty \mathbf{1}[\tau_1^- = T] \sum_{t = 1+ m(\delta)}^T \mathbf{1}[A_t \neq 1]\right]  \\
    &=\EE\left[\sum_{T = 1 + m(\delta)}^\infty  \sum_{t = 1+ m(\delta)}^T \mathbf{1}[\tau_1^- = T] \mathbf{1}[A_t \neq 1]\right]  \\
    &= \EE\left[ \sum_{t=1+m(\delta)}^\infty \mathbf{1}[\tau_1^- \geq t, A_t \neq 1] \right] \\ 
    &= \sum_{t=1+m(\delta)}^\infty \PP(\tau_1^- \geq t, A_t \neq 1) 
\end{align}

By Lemma \ref{lem:product_greater_joint}, $\PP(\tau_1^- \geq t,A_t \neq 1) \leq \PP(\tau_1^- \geq t) \PP(A_t \neq 1)$, and therefore we obtain the following bound for $(c)$:
\begin{align}
    (c) = \sum_{t=1+m(\delta)}^\infty \PP(\tau_1^- \geq t, A_t \neq 1)  \leq \frac{c}{m(\delta)} \sum_{t=1+m(\delta)}^\infty \PP(\tau_1^- \geq t) = \frac{c}{m(\delta)} \EE[\tau_1^-]. 
\end{align}




By plugging in $m(\delta)$ for expression $(a)$, we obtain the following lower bound:
\begin{align}
    (a) &\geq \EE[\tau_1^-]\left(1 - \frac{c}{m(\delta)}\right) - c\left(\log(m(\delta)) + 1 \right) \\
    &= \EE[\tau_1^-]\left(1 - c\frac{\log\left(1 + b\max\left(\frac{\xi}{1-\xi}, \frac{1-\xi}{\xi}\right)\right)}{\log(2K/\delta)}\right) - c\log\left(\frac{\log(2K/\delta)}{\log\left(1 + b\max\left(\frac{\xi}{1-\xi}, \frac{1-\xi}{\xi}\right)\right)}\right) -c\\ 
    &= \EE[\tau_1^-]\left(1 - O\left(\frac{1}{\log(1/\delta)}\right)\right) - O(\log\log(1/\delta)).
\end{align}

\paragraph{Step 3: Asymptotic Analysis.} Having bounds on terms $(a)$ and $(b)$ in Step 2, we now construct the desired result by taking the limit $\alpha \rightarrow 0$. Combining our bounds, we obtain:
\begin{align}
    (a)\times (b) \geq  \left(\EE[\tau_1^-]\left(1 - O\left(\frac{1}{\log(1/\delta)}\right)\right) - O(\log\log(1/\delta)) \right) \times \left( d(\mu_1, \xi) - O\left(\sqrt{\frac{\log(f(\delta))}{f(\delta)}}\right) \right).
\end{align}
By lines \ref{lines:final_steps}, we have an upper bound for the LHS:
\begin{align}
    \log(2K/\delta) + \log(1 + b\max(\frac{\xi}{1-\xi}, \frac{1-\xi}{\xi})) = \log(2K) + \log(1/\delta) +  \log(1 + b\max(\frac{\xi}{1-\xi}, \frac{1-\xi}{\xi})) \geq (a) \times (b).
\end{align}
By rearranging terms, we obtain the following inequality:
\begin{align}
    &\frac{\left(\EE[\tau_1^-](1 - O(\frac{1}{\log(1/\delta)})) - O(\log\log(1/\delta)) \right) \times \left( 1 - \frac{1}{d(\mu_1, \xi)}O\left(\sqrt{\frac{\log(f(\delta))}{f(\delta)}}  \right) \right)}{\log(1/\delta)} \leq \\
    &\quad \frac{\log(2K) + \log(1 + b\max(\frac{\xi}{1-\xi}, \frac{1-\xi}{\xi})) }{d(\mu_1, \xi) \log(1/\delta)} + \frac{1}{d(\mu_1, \xi)}
\end{align}
By taking the $\limsup$ as $\delta \rightarrow 0$ on both sides, and by the fact that $f(\delta) \rightarrow \infty$ when $\delta \rightarrow 0$, 
\begin{align}
    &\lim_{\delta\rightarrow 0}\sup_{P \in \Pcal(\bm\mu)}\frac{\left(\EE[\tau_1^-](1 - O(\frac{1}{\log(1/\delta)})) - O(\log\log(1/\delta)) \right) \times \left( 1 - \frac{1}{d(\mu_1, \xi)}O\left( \sqrt{\frac{\log(f(\delta))}{f(\delta)}} \right) \right)}{\log(1/\delta)} \\
    &= \lim_{\delta\rightarrow 0}\sup_{P \in \Pcal(\bm\mu)} \frac{\EE[\tau_1^-]}{\log(1/\delta)} \leq \frac{1}{d(\mu_1, \xi)} = \frac{1}{\log\frac{1-\mu_1}{1-\xi} + \mu_1 \log\frac{\mu_1(1-\xi)}{\xi(1-\mu_1)}} \label{line:final_result}.
\end{align}

Finally, note that $\EE[\tau_{\Gcal, 1}] \leq \EE[\tau_{1}^-]$ by definition of $\tau_{\Gcal, 1}$, so we obtain our desired result for the asymptotic stopping time of identifying one good arm. 

\paragraph{Step 4: Beyond $\tau_{\Gcal , 1}$.} We can reiterate this argument as follows, under the following modifications for Algorithm \ref{alg:i-GAI}. In lines 8-9 of Algorithm \ref{alg:i-GAI}, we additionally reset $E_t(\Hcal_a^-, b) = E_t(\Hcal_a^+, b) = 1$,  $\hat\mu_{t-1}(a) = \xi$, $N_{t}(a) = 0$ for all $a \in \Ical_t$. This effectively restarts the sampling scheme $\pi$ and testing procedures as if we had collected no information up to time $\tau_{\Gcal, 1}$. We emphasize that this is for analytical convenience for analyzing the limiting stopping times $\tau_{\Gcal, i}$ and $\tau_{\text{stop}}$. In practice, discarding such information is likely to cause far worse performance than the empirical results using Algorithm \ref{alg:i-GAI} in the main body of the paper. When space permits, we will add these comments to the main body of the paper. 

Given these modifications to Algorithm \ref{alg:i-GAI}, we immediately get our desired results. We now examine the case where there exists at least two good arms. First, note that $\tau_{\Gcal, 2}$ can be rewritten as follows:
\begin{align}
    \tau_{\Gcal, 2}  = (\tau_{\Gcal, 2} - \tau_{\Gcal, 1}) + \tau_{\Gcal, 1}
\end{align}
By resetting our algorithm after $\tau_{\Gcal, 1}$, the analysis for the term $(\tau_{\Gcal, 2} - \tau_{\Gcal, 1})$ is identical to the analysis for $\tau_{\Gcal, 1}$, where the summations are now taken from time indices $\tau_1^-< t < \tau_{\Gcal,2}$. The only difference is for time indices $\tau_1^-< t < \tau_{\Gcal,2}$, there exists only $K-1$ arms. When all $K$ arms were available (i.e. before any arms were labeled as good/bad), we could upper bound $\tau_{\Gcal,1}$ with $\tau_1^-$. With only $K-1$ arms, to attain the supremum (i.e. worst-case stopping time with $K-1$ arms), w.l.o.g., we assume that arm 1 (i.e. arm with the largest mean) was eliminated first, leaving us with the next best choice of $\tau_2^-$. Assuming that arm 1 was eliminated first results in a larger asymptotic lower bound in line \ref{line:final_result} for $\tau_{\Gcal, 2} - \tau_{\Gcal, 1}$ by the fact that $\frac{1}{d(\mu_2, \xi)} > \frac{1}{d(\mu_1, \xi)}$. We repeat this argument for all $i \in [G]$, where $G = |\Gcal_{\bm\mu}|$ is the number of good arms, to obtain the desired asymptotic result for $\tau_{\Gcal, i}$:
\begin{align}
    \lim_{\delta\rightarrow 0}\sup_{P \in \Pcal(\bm\mu)} \frac{\EE[\tau_{\Gcal, i}]}{\log(1/\delta)} &= \lim_{\delta\rightarrow 0}\sup_{P \in \Pcal(\bm\mu)} \left(\sum_{j=1}^{i-1}\frac{\EE[\tau_{\Gcal, j+1}- \tau_{\Gcal, j}] }{\log(1/\delta)} + \frac{\EE[\tau_{\Gcal, 1}]} {\log(1/\delta)}\right)  \\
    &\leq \lim_{\delta \rightarrow 0} \left(\sum_{j=1}^{i-1}\sup_{P \in \Pcal(\bm\mu)}\frac{\EE[\tau_{\Gcal, j+1}- \tau_{\Gcal, j}] }{\log(1/\delta)} + \sup_{P \in \Pcal(\bm\mu)}\frac{\EE[\tau_{\Gcal, 1}]}{\log(1/\delta)}\right) \label{line:sup_line} \\
    &\leq \sum_{j=1}^i \frac{1}{d(\mu_j, \xi)}.
\end{align}
where the inequality in line \ref{line:sup_line} follows from the fact that $\sup\sum_i{x_i} \leq \sum_i \sup(x_i)$. The proof for $\tau_{\text{stop}}$ follows analogously, where once we reach $K-G$ arms, we pick $\tau_{K-G}^+$, $\tau_{K-G + 1}^+$, ... $\tau_{K}^+$ to obtain the desired result.

\end{document}